\documentclass{article}
\usepackage[margin=1in]{geometry}

\usepackage{cite}
\usepackage{amsmath,amssymb,amsfonts}
\usepackage{algorithmic}
\usepackage{graphicx}
\usepackage{textcomp}

\usepackage{comment}

\usepackage{float}
\usepackage{booktabs}
\usepackage{multirow}
\usepackage{subcaption}

\usepackage{mathtools}
\DeclareMathOperator{\relu}{ReLU}

\usepackage{siunitx}

\usepackage{amsthm}
\theoremstyle{plain}
\newtheorem{theorem}{Theorem}
\newtheorem{proposition}{Proposition}

\newtheorem{lemma}{Lemma}
\theoremstyle{definition}
\newtheorem{definition}{Definition}
\newtheorem{example}{Example}
\newtheorem{assumption}{Assumption}
\theoremstyle{remark}
\newtheorem{remark}{Remark}

\theoremstyle{plain}
\makeatletter
\newtheorem*{rep@theorem}{\rep@title}
\newcommand{\newreptheorem}[2]{%
\newenvironment{rep#1}[1]{%
 \def\rep@title{#2 \ref{##1}}%
 \begin{rep@theorem}}%
 {\end{rep@theorem}}}
\makeatother
\newreptheorem{theorem}{Theorem}
\newreptheorem{proposition}{Proposition}
\newreptheorem{lemma}{Lemma}

\usepackage{hyperref}		% custom hyperlinks
\hypersetup{%
	colorlinks=true,%
	linkcolor=black,%
	urlcolor=black,%
	citecolor=black,%
	filecolor=black%
}							% hyperlink style settings
\usepackage[all]{hypcap}	% hyperlink to figure, not caption

\begin{document}

%%%%% FRONT MATTER

\title{Data-Driven Certification of Neural Networks \\
with Random Input Noise}
\author{Brendon G.\ Anderson\thanks{B.\ G.\ Anderson is with the Department of Mechanical Engineering, University of California, Berkeley (email: \href{mailto:bganderson@berkeley.edu}{bganderson@berkeley.edu}).} \and Somayeh Sojoudi\thanks{S.\ Sojoudi is with the Department of Electrical Engineering and Computers Sciences and the Department of Mechanical Engineering, University of California, Berkeley (email: \href{mailto:sojoudi@berkeley.edu}{sojoudi@berkeley.edu}).}
}
\date{}

\maketitle

\begin{abstract}
	Methods to certify the robustness of neural networks in the presence of input uncertainty are vital in safety-critical settings. Most certification methods in the literature are designed for adversarial or worst-case inputs, but researchers have recently shown a need for methods that consider random input noise. In this paper, we examine the setting where inputs are subject to random noise coming from an arbitrary probability distribution. We propose a robustness certification method that lower-bounds the probability that network outputs are safe. This bound is cast as a chance-constrained optimization problem, which is then reformulated using input-output samples to make the optimization constraints tractable. We develop sufficient conditions for the resulting optimization to be convex, as well as on the number of samples needed to make the robustness bound hold with overwhelming probability. We show for a special case that the proposed optimization reduces to an intuitive closed-form solution. Case studies on synthetic, MNIST, and CIFAR-10 networks experimentally demonstrate that this method is able to certify robustness against various input noise regimes over larger uncertainty regions than prior state-of-the-art techniques.
\end{abstract}

%%%%% BODY

% Introduction.

\section{Introduction}
\label{sec: introduction}
Real-world data is inherently uncertain. Such uncertainty comes in a variety of forms, including random measurement noise, adversarial attacks, and even structural perturbations in the underlying graph topology of networked systems \cite{franceschi2018robustness,jin2020boundary,geisler2020reliable}. Despite their excellent performance in a variety of decision and control tasks, e.g., distributed control using graph neural networks \cite{gama2021graph}, researchers have found that neural networks are highly sensitive to uncertainties in their inputs \cite{szegedy2014intriguing,fawzi2016robustness,su2019one}. This sensitive behavior is intolerable when using neural networks to operate safety-critical control systems, such as the power grid \cite{kong2017short}. As a result, a large emphasis has been placed by researchers on the development of methods that certify the robustness properties of neural networks.

Much of the literature on robustness certification has revolved around adversarial inputs, i.e., inputs with small-magnitude perturbations that are designed to cause a worst-case prediction \cite{wong2018provable,weng2018towards,raghunathan2018semidefinite,anderson2020tightened}. However, as argued in \cite{webb2019statistical} and \cite{mangal2019robustness}, random input uncertainty better models reality in many applications. Areas that commonly use a probabilistic model of uncertainty include stochastic control and finance, where unpredictable measurement errors and state disturbances are assumed to be random \cite{aastrom2012introduction,follmer2016stochastic}. The stochastic framework also naturally encapsulates applications where unbounded uncertainties may exist, albeit with an extremely low probability. This is typical in real-world applications such as aviation \cite{zakrzewski2004randomized}. In fact, the International Organization for Standardization (ISO) asserts in their guide on safety aspects that there is never absolute safety, and therefore the goal is to achieve what they define to be \emph{tolerable risk} \cite{huang2021statistical,iso1999safety}. Not only are random uncertainties pervasive and realistic, they have been shown to pose a legitimate threat---small uniform noise causes misclassification rates of well over $10\%$ on MNIST and CIFAR-10 networks, and for Bernoulli noise the misclassification rates become drastically worse, sometimes reaching $100\%$ \cite{weng2019proven}.

The aforementioned motivations have led to an influx of recent works considering robustness against random inputs, which we review in Section \ref{sec: related_works}. Many of them make stringent assumptions on the structure of the network or input distribution, or the formal certification guarantees are relaxed or eliminated in order to enhance computational tractability. Since neural networks are more sensitive to adversarial inputs than to random input noise \cite{franceschi2018robustness}, worst-case sensitivity analyses are too conservative for random input noise when the goal at hand is to achieve a tolerable risk level, whereas high-probability robustness certificates may completely fail in the presence of adversaries; the two settings are disjoint and should be studied using distinct methods. Consequently, our study is intended to certify robustness against random input noise with minimal conservatism, and is not intended to assess adversarial robustness.

% Related Works.

\subsection{Related Works}
\label{sec: related_works}

In this section, we review the state-of-the-art methods for assessing robustness to random inputs, highlight their usages, and address their limitations. For instance, \cite{mangal2019robustness} defines robustness as the network output being Lipschitz continuous with high probability when two inputs are chosen randomly. Their proposed method is limited to neural networks composed of conditional affine transformations, e.g., ReLU networks. On the other hand, \cite{weng2019proven} analytically bounds the probability that a classifier's margin function exceeds a given value. Although this probabilistic method applies to general neural network models, it assumes that the random input noise is constrained to an $\ell_p$-norm ball and is either Gaussian or has independent coordinates. Furthermore, their bounding technique relies on worst-case analysis methods, making their resulting certificates relatively loose (see Section \ref{sec: numerical_experiments}).

In \cite{webb2019statistical}, robustness is measured by the probability that random input noise results in misclassification. The authors propose a sampling-based approximation of the robustness level. However, no theoretical guarantees are given to certify the network's robustness. Contrarily, \cite{dvijotham2018verification} formally bounds the probability that a random input maps to an unsafe output. However, the bounding function is nonconvex, and therefore obtaining tight bounds amounts to nonconvex optimization. Alternatively, \cite{franceschi2018robustness} bounds the size of a random input perturbation that causes a classifier prediction error with high probability. Their bounds provide elegant theoretical guarantees, but they depend on the network's worst-case robustness level, which is generally NP-hard to compute without approximation errors or additional assumptions \cite{weng2018towards,katz2017reluplex}.

The works \cite{fazlyab2019probabilistic} and \cite{couellan2021probabilistic} provide methods to guarantee that network outputs do not significantly deviate from the nominal output when the input is subject to random uncertainty. This corresponds to the problem of localizing the network outputs within the output space. The work \cite{fazlyab2019probabilistic} uses the output localization to issue high-probability guarantees for the network's robustness. However, their method requires solving a semidefinite program, and their results are demonstrated on small, single-layer networks, so it is not clear whether their method scales to realistic applications. The concentration bounds presented in \cite{couellan2021probabilistic} can be applied to deep networks, but their localization results do not immediately translate into a meaningful certificate of robustness.

The authors of \cite{devonport2020estimating} and \cite{li2021probabilistic} use input-output samples to learn how input noise is propagated to the output space through a method called scenario optimization. This approach naturally embeds the stochastic nature of the input noise into the assessment procedure. The work \cite{devonport2020estimating} provides a method to estimate a network's set of possible outputs, but this localization may fail to determine the safety of the output since the underlying optimization does not directly consider any safety specifications, e.g., classification boundaries. We demonstrate this phenomenon in Section \ref{sec: comparison_to_output_set_estimation}. Other scenario-based output set estimation techniques for general nonlinear maps, such as \cite{yang2014stochastic,fravolini2017probabilistic,sartipizadeh2019voronoi}, suffer from the same limitation in the context of neural network certification. Contrarily, \cite{li2021probabilistic} directly considers output safety in their scenario approach. However, their method makes use of an affine approximation to the network's nonlinear margin function, a worst-case analysis technique from the adversarial robustness literature. In our experiments, we show that this worst-case technique yields loose robustness bounds as the size of the input noise increases.

Finally, \cite{zakrzewski2004randomized} also uses sampled data to bound the probability of failure. Their guarantees take the probably approximately correct (PAC) form in terms of probability levels $\epsilon,\delta\in[0,1]$, and they improve the sample complexity from $O(\frac{1}{\epsilon^2}\log\frac{1}{\delta})$ of the naive Chernoff bound to $O(\frac{1}{\epsilon}\log\frac{1}{\delta})$. This is achieved by imposing a Bayesian framework and assuming that the failure probability follows a uniform prior distribution. As the authors remark, this is a ``very conservative choice.'' In contrast, the method to be proposed in this paper solves for the optimal (least conservative) robustness bound of this form with the same $O(\frac{1}{\epsilon}\log\frac{1}{\delta})$ sample complexity.

% Contributions.

\subsection{Contributions}
\label{sec: contributions}

In this paper, we develop a data-driven framework for certifying neural network robustness against random input noise using scenario optimization. Our direct approach avoids worst-case analysis techniques, such as those found in the adversarial robustness literature, e.g., \cite{weng2019proven}, and also avoids the need for selecting a Bayesian prior distribution governing the network's failure probability, as in \cite{zakrzewski2004randomized}. The procedure is capable of localizing network outputs into a general class of sets, and we develop sufficient conditions on this class to ensure that the procedure amounts to a convex optimization problem. Furthermore, we develop formal guarantees that the resulting robustness certificate holds with overwhelming probability upon using sufficiently many samples in the scenario optimization. Although the method is applicable to all neural networks and all input noise distributions, we show how to exploit the structure of networks with affinely bounded activation functions in order to reduce sample complexity.

	Our numerical experiments demonstrate that the proposed optimization is capable of issuing robustness certificates in cases where the two-step process of optimally localizing the outputs (e.g., using \cite{devonport2020estimating}) and then certifying them cannot, providing a novel perspective that output set estimation techniques do not necessarily work well for certification. Furthermore, we show on both synthetic networks and large MNIST and CIFAR-10 networks that our robustness bounds are much tighter than those obtained by the state-of-the-art method \cite{weng2019proven}, particularly for large noise levels.

% Outline.
\subsection{Outline}
\label{sec: outline}
We begin by formulating the certification problem in Section \ref{sec: problem_statement}. In Sections \ref{sec: formulating_the_certificate} and \ref{sec: data-driven_reformulation}, we propose an optimization to solve for the high-probability robustness certificate, and then show that it suffices to solve a data-driven convex optimization problem with sufficiently many samples. Next, we demonstrate how to exploit the neural network structure to reduce the sample complexity of the proposed method in Section \ref{sec: exploiting_network_structure}. We numerically illustrate the results and compare to state-of-the-art methods in Section \ref{sec: numerical_experiments}. We conclude in Section \ref{sec: conclusions}. Proofs and various supplementary materials are presented in the appendices.

% Notations.
\subsection{Notations}
\label{sec: notations}

The ceiling of $x\in\mathbb{R}$ is written $\lceil x \rceil$. For $x,y\in\mathbb{R}^n$, we define $[x,y] = \{z\in\mathbb{R}^n : x\le z\le y\}$, where the inequalities are interpreted element-wise. Given a set $\mathcal{X}$, we denote its power set by $\mathcal{P}(\mathcal{X})$. The Minkowski sum of sets $\mathcal{X}$ and $\mathcal{Y}$ is defined as $\mathcal{X}+\mathcal{Y}=\{x+y : x\in\mathcal{X},~y\in\mathcal{Y}\}$. We define $\mathbb{R}_{++} = \{x\in\mathbb{R} : x>0\}$. For a function $f\colon\mathbb{R}^m\to\mathbb{R}^n$, we write the image of $\mathcal{X}\subseteq\mathbb{R}^m$ under $f$ as $f(\mathcal{X})=\{f(x)\in\mathbb{R}^n : x\in\mathcal{X}\}$. If $g\colon\mathbb{R}^n\to\mathbb{R}^p$ is another function, we define the composition $g\circ f\colon\mathbb{R}^m\to\mathbb{R}^p$ by $g\circ f(x) = g(f(x))$. If $X\colon\Omega\to\mathbb{R}^n$ is a random variable on a probability space $(\Omega,\mathcal{F},\mathbb{P})$, $g\colon\mathbb{R}^n\to\mathbb{R}$ is a Borel measurable function, and $c\in\mathbb{R}$, we use the notation $\mathbb{P}(g(X)\ge c)$ to mean $\mathbb{P}(\{\omega\in\Omega : g(X(\omega)) \ge c\})$. Similarly, if $\mathcal{S}\subseteq\mathbb{R}^m$ is a Borel set and $h\colon\mathbb{R}^n\to\mathbb{R}^m$ is a Borel measurable function, we write $\mathbb{P}(h(X)\in \mathcal{S})$ to mean $\mathbb{P}(\{\omega\in \Omega : h(X(\omega))\in\mathcal{S}\})$. For a norm $\|\cdot\|$ on $\mathbb{R}^n$ we denote its dual norm by $\|\cdot\|_*$, where $\|y\|_* = \sup \{x^\top y : \|x\| \le 1\}$. We assume throughout that optimization problems are attained by a solution.

% Problem Statement.

\section{Problem Statement}
\label{sec: problem_statement}

\subsection{Network Description, Safe Set, and Safety Level}
\label{sec: network_description_safe_set_safety_level}

In this paper, we consider a Borel measurable neural network $f\colon\mathbb{R}^{n_x}\to\mathbb{R}^{n_y}$ with arbitrary structure and parameters.\footnote{Borel measurability of the network is almost always satisfied in practice. Indeed, every continuous function is Borel measurable.} We assume that the input to the network is a random variable $X\colon\Omega_X\to\mathbb{R}^{n_x}$ on a fixed probability space $(\Omega_X,\mathcal{F}_X,\mathbb{P}_X)$.\footnote{It is easily verified that all probabilistic expressions in this paper are well-defined from a measure-theoretic perspective. We leave out these measurability verifications for the sake of exposition.} We do not assume that the distribution $\mathbb{P}_X$ is exactly known---we only assume that we are able to sample from $\mathbb{P}_X$. The support of the probability measure $\mathbb{P}_X$ is called the \emph{input set}, which is denoted by $\mathcal{X}\subseteq\mathbb{R}^{n_x}$. The \emph{output set} of the network is defined to be $\mathcal{Y}=f(\mathcal{X})\subseteq\mathbb{R}^{n_y}$.

Next, consider a given convex polyhedral \emph{safe set} $\mathcal{S}=\{y\in\mathbb{R}^{n_y} : Ay+b\ge 0\}$, where $A\in\mathbb{R}^{n_s\times n_y}$ and $b\in\mathbb{R}^{n_s}$. Without loss of generality, we assume that $n_s = 1$, henceforth setting $A=a^\top\in\mathbb{R}^{1\times n_y}$ and $b\in\mathbb{R}$.\footnote{The polyhedral safe set assumption is without loss of generality. Suppose that the safe set is $\mathcal{S} = \{y\in\mathbb{R}^{n_y} : a^\top g(y) + b \ge 0\}$ for some nonlinear Borel measurable $g\colon\mathbb{R}^{n_y}\to\mathbb{R}^{n_z}$ and some $a\in\mathbb{R}^{n_z},b\in\mathbb{R}$. Then we can reduce the problem to our assumed form by considering $f' \coloneqq g\circ f$ to be the (Borel measurable) neural network and $\mathcal{S}' \coloneqq \{z\in \mathbb{R}^{n_z} : a^\top z + b \ge 0\}$ to be the (polyhedral) safe set, since then $f(x)\in\mathcal{S}$ if and only if $f'(x)\in \mathcal{S}'$, for all $x\in\mathcal{X}$. The architecture-dependent results of Section \ref{sec: exploiting_network_structure} must be applied to $f'$ with care, since $g$ must now be considered as another layer in the network, making Assumptions \ref{ass: preactivation_bounds} and \ref{ass: affine_activation_bounds} to follow slightly more stringent.} The results of this paper can be immediately generalized to the case where $n_s>1$. See Appendix \ref{sec: extension_to_general_polyhedral_safe_sets} for a detailed explanation.

The elements of the set $\mathcal{S}$ are considered to be \emph{safe}. For a point $y$ in the output space $\mathbb{R}^{n_y}$, the value $s(y) = a^\top y + b$ is called the \emph{safety level} of $y$. The point $y$ is safe if and only if its safety level is nonnegative. Broadly speaking, the overall goal of this paper is to certify that the random output $Y=f(X)$ is safe. When this holds for all or many of the possible outputs in $\mathcal{Y}$, we obtain a natural certificate for the robustness of the network against the random noise.

\begin{example}
\label{ex: classification_network_setup}
When $f$ is an $n_y$-class classifier and $\bar{x}\in\mathbb{R}^{n_x}$ is a deterministic nominal input with class $i^*\in{\arg\max}_{i\in\{1,2,\dots,n_y\}} f_i(\bar{x})$, a common goal is to certify that additive random noise $\delta$ on $\bar{x}$ does not cause misclassification \cite{weng2019proven}. This problem falls within our framework by  defining the safety level of $f(X)$ to be the \emph{margin function} value $g_i(X) \coloneqq f_{i^*}(X) - f_i(X)$ for $X=\bar{x}+\delta$.
\end{example}

\subsection{Various Notions of Robustness}
\label{sec: various_notions_of_robustness}

We now use the safety level of outputs to introduce three meaningful notions of robustness against random input noise, and discuss how they are related to one another.

\subsubsection{Deterministic Robustness Level}
\label{sec: deterministic_robustness_level}

The \emph{deterministic robustness level} of the network is defined as
\begin{equation}
	r^* = \inf_{y\in\mathcal{Y}} a^\top y + b. \label{eq: deterministic_robustness_level}
\end{equation}
If $r^*\ge 0$, then $\mathcal{Y}\subseteq \mathcal{S}$, implying that the random output $Y=f(X)$ is safe with probability one. This notion of robustness coincides with that used when considering adversarial inputs \cite{wong2018provable,raghunathan2018semidefinite,anderson2020tightened}, but the resulting worst-case safety level is often much lower than the safety levels of random outputs in practice \cite{webb2019statistical}. Consequently, using $r^*$ to assess robustness may falsely indicate that the network is sensitive to the input noise.

\subsubsection{Approximate Robustness Level}
\label{sec: approximate_robustness_level}

Although $r^*$ can issue strong guarantees about the safety of the network output, \eqref{eq: deterministic_robustness_level} amounts to an intractable nonconvex optimization problem, since $\mathcal{Y}$ is generally a nonconvex set. Instead of computing $r^*$, we can consider approximating it by
\begin{equation}
	\hat{r}(\hat{\mathcal{Y}}) = \inf_{y\in\hat{\mathcal{Y}}} a^\top y + b, \label{eq: approximate_robustness_level}
\end{equation}
where $\hat{\mathcal{Y}}\subseteq\mathbb{R}^{n_y}$, termed the \emph{surrogate output set}, is more tractable than $\mathcal{Y}$, and preferably convex. We call \eqref{eq: approximate_robustness_level} the \emph{approximate robustness level} of the network. If $\mathcal{Y}\subseteq\hat{\mathcal{Y}}$, then $\hat{r}(\hat{\mathcal{Y}})\le r^*$. In this case, if $\hat{r}(\hat{\mathcal{Y}}) \ge 0$, then the random output $Y=f(X)$ is safe with probability one. In general, choosing $\hat{\mathcal{Y}}$ to cover $\mathcal{Y}$ makes $\hat{r}(\hat{\mathcal{Y}})$ an over-conservative measure of robustness for the same reasons $r^*$ is.

\subsubsection{Probabilistic Robustness Level}
\label{sec: probabilistic_robustness_level}

The notion of deterministic robustness is too strong for applications involving random input noise, as many input distributions have unbounded support or have their worst-case inputs in regions of low probability measure \cite{webb2019statistical}. Furthermore, \eqref{eq: deterministic_robustness_level} and \eqref{eq: approximate_robustness_level} neglect the distributional information given for $X$. This conservatism means that the robustness levels \eqref{eq: deterministic_robustness_level} and \eqref{eq: approximate_robustness_level}, with $\mathcal{Y}\subseteq\hat{\mathcal{Y}}$, are generally unable to certify that $Y$ is safe, even when $Y$ concentrates around safe outputs. Consequently, for a prescribed probability level $\epsilon\in[0,1]$, we define the more natural \emph{probabilistic robustness level} of the network to be
\begin{equation}
	\bar{r}(\epsilon) = \sup\{ r\in\mathbb{R} : \mathbb{P}_X(a^\top f(X) + b \ge r) \ge 1-\epsilon \}. \label{eq: probabilistic_robustness_level}
\end{equation}
Intuitively, the random output $f(X)$ has a safety level at least $\bar{r}(\epsilon)$ with high probability. In the case that $\bar{r}(\epsilon)\ge 0$, we certify that the random output $Y=f(X)$ is safe with probability $1-\epsilon$, and we say that the network is \emph{probabilistically robust}.\footnote{Note the distinction: ``safety'' is a property of outputs, whereas ``robustness'' is a property of the neural network (with respect to the noisy input distribution $\mathbb{P}_X$). We are always careful when using these terminologies.} Another interpretation of probabilistic robustness is that the majority of possible outputs (with respect to the distribution $\mathbb{P}_X$) are safe. The probabilistic robustness level is catered towards our setting of random input noise, and, compared to the worst-case alternatives, reduces conservatism by considering the actual likelihood of the possible inputs. This is precisely the notion of robustness we adopt in this paper. We remark that this definition of probabilistic robustness coincides with that used in \cite{webb2019statistical} and \cite{weng2019proven}, albeit our subsequent analysis and guarantees vastly differ from these works.

\begin{example}
\label{ex: classification_network_probabilistic_robustness_level}
Consider again the classification network in Example \ref{ex: classification_network_setup}. In this setting, if $r^*\ge 0$, then the probability of misclassification is zero, in which case the noise $\delta$ is not important. On the other hand, if there exists a perturbed input $\bar{x}+\delta$ in the input set $\mathcal{X}$ that is misclassified, then $r^*<0$. It may be the case, however, that this misclassification occurs with a sufficiently low probability with respect to the tolerance $\epsilon$. This is precisely what we seek to certify: that the probability of misclassification is less than $\epsilon$, which mathematically amounts to showing that $\bar{r}(\epsilon)\ge 0$.
\end{example}

In this paper, we aim to certify the safety of $Y=f(X)$ by lower-bounding the probabilistic robustness level $\bar{r}(\epsilon)$. A trivial lower bound is easily verified: $r^*\le \bar{r}(\epsilon)$ for all $\epsilon\in[0,1]$, and $r^* = \bar{r}(0)$. However, as we will see in Section \ref{sec: numerical_experiments}, this approach of considering \emph{worst-case} inputs instead of \emph{likely} inputs, often results in a very loose lower bound, sometimes failing to issue a robustness guarantee at all. In the next section, we show that by using a special type of surrogate output set in $\hat{r}(\hat{\mathcal{Y}})$, we can optimize a lower bound on $\bar{r}(\epsilon)$ and obtain an estimate of the output set $\mathcal{Y}$ as a natural byproduct.

% Formulating the Certificate.

\section{Formulating the Certificate}
\label{sec: formulating_the_certificate}

\subsection{Bounding the Probabilistic Robustness Level}
\label{sec: bounding_the_probabilistic_robustness_level}

As we have seen, $\hat{r}(\hat{\mathcal{Y}})\approx r^* \le \bar{r}(\epsilon)$. Two natural questions arise. 1) Can one use $\hat{r}(\hat{\mathcal{Y}})$ to certifiably lower-bound the quantity $\bar{r}(\epsilon)$ of interest? 2) If so, how can $\hat{\mathcal{Y}}$ be chosen to optimize the bound? In this section, we study the first question. As it turns out, such a lower bound holds so long as $\hat{\mathcal{Y}}$ has high enough coverage over $\mathcal{Y}$. Before proving this claim, we formally define this notion of coverage.

\begin{definition}
	\label{def: epsilon-cover}
	Let $\hat{\mathcal{Y}}$ be a subset of $\mathbb{R}^{n_y}$. For $\epsilon\in[0,1]$, the set $\hat{\mathcal{Y}}$ is said to be an \emph{$\epsilon$-cover} of $\mathcal{Y}=f(\mathcal{X})$ if $\hat{\mathcal{Y}}$ is Borel measurable and $\mathbb{P}_X(f(X)\in\hat{\mathcal{Y}}) \ge 1-\epsilon$.
\end{definition}

Intuitively, an $\epsilon$-cover of $\mathcal{Y}$ is a set that contains $Y=f(X)$ with high probability. If we can compute an $\epsilon$-cover of $\mathcal{Y}$, then we will have localized the output with high confidence. By restricting $\hat{\mathcal{Y}}$ in \eqref{eq: approximate_robustness_level} to be an $\epsilon$-cover of $\mathcal{Y}$, we ensure that the approximate robustness level takes into account the \emph{likely} inputs $X$, but not necessarily the worst-case inputs. Consequently, this permits $\hat{r}(\hat{\mathcal{Y}})$ to be greater than $r^*$, reducing the conservatism in our measure of robustness caused by unlikely worst-case inputs. We now show that this special type of surrogate output set is a good enough estimate of $\mathcal{Y}$ to maintain the lower bound on $\bar{r}(\epsilon)$ that we seek.

\begin{proposition}
	\label{prop: lower_bound_from_epsilon-cover}
	Let $\hat{\mathcal{Y}}$ be an arbitrary subset of $\mathbb{R}^{n_y}$. If $\hat{\mathcal{Y}}$ is an $\epsilon$-cover of $\mathcal{Y}=f(\mathcal{X})$, then
	\begin{equation}
		\hat{r}(\hat{\mathcal{Y}}) \le \bar{r}(\epsilon). \label{eq: lower_bound_from_epsilon-cover}
	\end{equation}
\end{proposition}
	\begin{proof}
		See Appendix \ref{app: proofs}.
	\end{proof}

\begin{comment}
% Commenting out the following paragraph since it is not very insightful and we need space in the TCNS submission.
Proposition \ref{prop: lower_bound_from_epsilon-cover} can be interpreted as follows. Suppose that $\hat{\mathcal{Y}}$ is chosen to be an $\epsilon$-cover of $\mathcal{Y}$ and the approximate robustness level, $\hat{r}(\hat{\mathcal{Y}})$, is computed using $\hat{\mathcal{Y}}$ as the surrogate output set. Then each of the following two events hold with probability at least $1-\epsilon$: 1) the random output $Y=f(X)$ of the neural network has a safety level at least $\hat{r}(\hat{\mathcal{Y}})$, and 2) the output $Y$ is contained in $\hat{\mathcal{Y}}$. In particular, if $\hat{r}(\hat{\mathcal{Y}})\ge 0$, then the random output $Y$ is safe with probability at least $1-\epsilon$. The proposition thereby shows that the approximate robustness level can be used for certification and localization of the output so long as the surrogate output set is chosen appropriately.
\end{comment}

\subsection{Optimizing the Bound}
\label{sec: optimizing_the_bound}

From Proposition \ref{prop: lower_bound_from_epsilon-cover}, we know that $\epsilon$-covers constitute good choices of the surrogate output set $\hat{\mathcal{Y}}$ used to compute the approximate robustness level. This is because the random output $Y=f(X)$ is guaranteed to have safety level at least $\hat{r}(\hat{\mathcal{Y}})$ with high probability. However, it is entirely possible that the choice of $\epsilon$-cover results in $\hat{r}(\hat{\mathcal{Y}})<0$, even when the network is probabilistically robust. In this case, $\hat{r}(\hat{\mathcal{Y}})$ fails to issue a high-probability certificate for the safety of the random output $Y=f(X)$, despite $\hat{\mathcal{Y}}$ being able to localize it.

To overcome the above problem, we turn to studying our second inquiry from earlier, namely, how to optimize the lower bound \eqref{eq: lower_bound_from_epsilon-cover}. This amounts to finding an $\epsilon$-cover of $\mathcal{Y}$ that maximizes $\hat{r}(\hat{\mathcal{Y}})$. Since optimizing over all subsets of $\mathbb{R}^{n_y}$ is intractable, we restrict our search to sets within a class $\mathcal{H} = \{h(\theta) : \theta\in\Theta\}$ parameterized by a parameter set $\Theta\subseteq \mathbb{R}^p$ and a set-valued function $h\colon\mathbb{R}^p \to \mathcal{P}(\mathbb{R}^{n_y})$. We assume throughout that the class $\mathcal{H}$ is chosen such that $h(\theta)$ is a Borel set for all parameters $\theta\in\Theta$. A concrete example of one such class is given below.

\begin{example}
	\label{ex: norm_ball_class}
	Let $\|\cdot\|$ be a fixed norm on $\mathbb{R}^{n_y}$ and $\Theta = \mathbb{R}^{n_y}\times\mathbb{R}_{++}$. Defining $p=n_y+1$, let $h\colon\mathbb{R}^{p}\to\mathcal{P}(\mathbb{R}^{n_y})$ be defined by $h(\bar{y},r) = \{y\in\mathbb{R}^{n_y} : \|y-\bar{y}\|\le r\}$. Then, $\Theta$ and $h(\cdot)$ define the class of $\|\cdot\|$-norm balls:
	\begin{equation*}
		\mathcal{H} = \big\{ \{y\in\mathbb{R}^{n_y} : \|y-\bar{y}\|\le r\} : r>0, ~ \bar{y}\in\mathbb{R}^{n_y} \big\}.
	\end{equation*}
\end{example}

The problem of choosing $h(\cdot)$ and $\Theta$ (and therefore also $\mathcal{H}$) is discussed in detail in Section \ref{sec: data-driven_reformulation}. By restricting our search for $\epsilon$-covers to within the class $\mathcal{H}$, our search reduces to maximizing the approximate robustness level over the parameter set $\Theta$. By slightly abusing notation, we denote the dependence of the approximate robustness level on the parameter $\theta$ explicitly as $\hat{r}(\theta) = \inf\{a^\top y + b : y\in h(\theta)\}$, and we formulate the following optimization problem:
\begin{equation}
	\begin{aligned}
	& \underset{\theta\in\Theta}{\text{maximize}} && \hat{r}(\theta) - \lambda v( \theta ) \\
	& \text{subject to} && \mathbb{P}_X(f(X)\in h(\theta)) \ge 1-\epsilon,
	\end{aligned} \label{eq: chance-constrained_problem}
\end{equation}
where $\lambda\ge 0$ and $v \colon \mathbb{R}^p\to\mathbb{R}$ is taken to be a nonnegative convex function on $\Theta$ that increases with the volume of $h(\theta)$. The objective $\hat{r}(\theta)$ in \eqref{eq: chance-constrained_problem} is the approximate robustness level computed using the set $h(\theta)$ as the surrogate output set. The constraint $\mathbb{P}_X(f(X)\in h(\theta)) \ge 1-\epsilon$ enforces that we only consider parameters $\theta$ such that $h(\theta)$ is an $\epsilon$-cover of $\mathcal{Y}$. The regularization term $-\lambda v(\theta)$ penalizes the size of $h(\theta)$. This makes the set $h(\theta)$ as small as possible while maintaining its $\epsilon$-coverage, thereby yielding the tightest localization of the output $Y$. The regularization is done at the expense of a slightly suboptimal bound \eqref{eq: lower_bound_from_epsilon-cover}, and can be eliminated by setting $\lambda=0$, if no localization of the output $Y$ is desired. On the other hand, increasing $\lambda$ amounts to putting more assessment effort into localizing $Y$, making the $\epsilon$-cover $h(\theta)$ a better estimate of $\mathcal{Y}$. This certification-localization tradeoff is experimentally explored in Section \ref{sec: illustrative_example}.

\begin{comment}
% This paragraph is in response to a specific ACC review, which asked why we are maximizing instead of minimizing.
To reiterate, we are maximizing the approximate robustness level $\hat{r}(\theta)$ in order to find the best lower bound on the probabilistic robustness level $\bar{r}(\epsilon)$. If we find that $\hat{r}(\theta)\ge 0$ for some $\theta$ feasible for \eqref{eq: chance-constrained_problem}, then by Proposition \ref{prop: lower_bound_from_epsilon-cover} we have certified that $\bar{r}(\epsilon)\ge 0$, so the output $Y=f(X)$ is safe with probability at least $1-\epsilon$. Because of this inherent goal to maximize the lower bound on $\bar{r}(\epsilon)$, changing the optimization \eqref{eq: chance-constrained_problem} to a minimization problem would be meaningless to our current study.
\end{comment}

% Data-Driven Reformulation.

\section{Data-Driven Reformulation}
\label{sec: data-driven_reformulation}

Even when the set $h(\theta)$ is convex for all $\theta\in\Theta$, the probabilistic constraint in \eqref{eq: chance-constrained_problem} is in general nonconvex \cite{nemirovski2007convex}. Constraints of this form are referred to as \emph{chance constraints}, and there exist various approaches to reformulating and relaxing them into convex constraints. Since the problem at hand considers neural networks whose models are usually complicated to analyze, but whose input-output samples are easily obtained, we seek a data-driven approach to approximately enforcing the chance constraint in \eqref{eq: chance-constrained_problem}, without losing the robustness certificate provided by the solution. The \emph{scenario approach} is a popular method within the stochastic optimization and robust control communities that replaces the chance constraint with hard constraints on a number of random samples \cite{nemirovski2007convex,tempo2012randomized,campi2009scenario,luedtke2008sample}. The scenario approach has been studied for general problems, even those with nonconvex objectives and those whose resulting hard constraints are nonconvex \cite{campi2018general}. However, the most powerful use of scenario optimization arises when the resultant scenario problem is convex, as then \emph{a priori} probabilistic guarantees can be made about the solution's feasibility for the original chance constraint. As we will soon see, this sampling-based method fits nicely into the framework of our problem, and maintains a lower bound on $\bar{r}(\epsilon)$ with high probability, provided that a sufficiently large number of samples is used and the scenario problem is convex.

To implement the scenario approach, suppose that $\{x_j : j\in\{1,2,\dots,N\}\} \subseteq \mathcal{X}$ is a set of $N$ independent and identically distributed samples drawn from $\mathbb{P}_X$. For each input $x_j$, we compute its corresponding output $y_j=f(x_j)$. Then, replacing the chance constraint in \eqref{eq: chance-constrained_problem} with $N$ hard constraints on the samples $y_j$ yields the following scenario optimization:
\begin{equation}
	\begin{aligned}
	& \underset{\theta\in\Theta}{\text{maximize}} & & \hat{r}(\theta) - \lambda v(\theta) \\
	& \text{subject to} & & y_j\in h(\theta) ~ \text{for all $j\in\{1,2,\dots,N\}$}.
	\end{aligned} \label{eq: scenario_problem}
\end{equation}
Note that, because the data $y_j$ is random, solutions $\theta^*$ to \eqref{eq: scenario_problem} are random. We assume throughout the paper that \eqref{eq: scenario_problem} is attained by a solution $\theta^*$, and we denote the probability space on which it is defined by $(\Omega_{\theta^*},\mathcal{F}_{\theta^*},\mathbb{P}_{\theta^*})$.

\begin{remark}
	The above assumption of independent and identically distributed samples is critical for relating the solution $\theta^*$ back to the original chance-constrained problem \eqref{eq: chance-constrained_problem}. In particular, it is a key assumption on which the forthcoming high-probability robustness certificate in Theorem \ref{thm: high-probability_guarantees} rests. Despite these assumptions holding in many practical models, the independence may be violated in certain applications with inherent time-correlation between samples, and the assumption prevents the use of selective sampling to improve the efficiency of the scenario approach.

	The identical distribution assumption is also critical, and it may be violated in two main ways. First, the underlying distribution of the data used in \eqref{eq: scenario_problem} may change from sample to sample, and second, the underlying noise distribution of the actual input may be different in practice from the samples used in the robustness certification procedure. Despite these sources of modeling error, our scenario-based approach can be modified into a distributionally robust variant to still give high-probability robustness certificates in the case that the distribution of the input is contained in a finite set of possible distributions.
\end{remark}

As mentioned in Section \ref{sec: related_works}, the scenario approach was used recently in reachable set estimation for dynamical systems \cite{devonport2020estimating}. We remark that \eqref{eq: scenario_problem} recovers the scenario optimization of \cite{devonport2020estimating} in the special case that the objective is re-scaled to $\frac{1}{\lambda}\hat{r}(\theta) - v(\theta)$ and $\lambda\to\infty$, the regularizer $v(\theta)$ equals the volume of the set $h(\theta)$, and $\mathcal{H}$ is the norm ball class. This reduction amounts to finding the tightest norm ball $\epsilon$-cover of $\mathcal{Y}$, without regard to optimizing the lower bound \eqref{eq: lower_bound_from_epsilon-cover} of interest. In Section \ref{sec: comparison_to_output_set_estimation}, we demonstrate the necessity for the more general formulation \eqref{eq: scenario_problem} by giving an example where reducing to the special case of \cite{devonport2020estimating} causes robustness certification to fail, despite finding the tightest $\epsilon$-cover of $\mathcal{Y}$.

% commenting paragraph out in tcns revision 2
\begin{comment}
Taking $\lambda=0$ in \eqref{eq: scenario_problem} amounts to using the sampled data purely for optimizing the bound \eqref{eq: lower_bound_from_epsilon-cover}, so no effort is placed into localizing the output $Y$. In Appendix \ref{sec: special_case-class_of_half-spaces}, we show that, under this setting, taking $\mathcal{H}$ to be the class of all half-spaces in $\mathbb{R}^{n_y}$ makes \eqref{eq: scenario_problem} coincide with directly applying scenario optimization to estimate $\bar{r}(\epsilon)$ via its definition \eqref{eq: probabilistic_robustness_level}. These observations illustrate the generality of our proposed certification problem, as it provides a unified framework for simultaneous certification \emph{and} localization of the random output $Y$.
\end{comment}

Although the scenario approach has eliminated the chance constraint from \eqref{eq: chance-constrained_problem}, there remain two problems to consider. First, it is not immediately clear whether \eqref{eq: scenario_problem} is convex or computationally tractable, as it has an inherent max-min optimization structure. However, it is important to ensure the problem's convexity, since no \emph{a priori} guarantees can be made regarding the feasibility of $\theta^*$ for the original chance constraint in the general case of nonconvex scenario optimization \cite{campi2018general}. In Section \ref{sec: conditions_for_convex_optimization}, we leverage results from parametric optimization to develop conditions on our choice of $\Theta$ and $h(\cdot)$ to ensure that \eqref{eq: scenario_problem} is convex. Second, the solution of \eqref{eq: scenario_problem} gives a random approximation to the solution of \eqref{eq: chance-constrained_problem}, which optimizes the bound \eqref{eq: lower_bound_from_epsilon-cover} on $\bar{r}(\epsilon)$. In Section \ref{sec: high-probability_guarantees} we develop formal guarantees showing that the solution of \eqref{eq: scenario_problem} maintains a lower bound on $\bar{r}(\epsilon)$ with high probability, provided that $N$ is sufficiently large.

% Conditions for Convex Optimization.

\subsection{Conditions for Convex Optimization}
\label{sec: conditions_for_convex_optimization}

In this section, we consider the choices of the parameter set $\Theta$ and the set-valued function $h(\cdot)$ on lower-bounding $\bar{r}(\epsilon)$, and on the tractability of the resulting scenario problem \eqref{eq: scenario_problem}. A key insight is this: an $\epsilon$-cover of $\mathcal{Y}$ may in general be much larger than $\mathcal{Y}$ itself. This is because regions of an $\epsilon$-cover that do not intersect with $\mathcal{Y}$ also do not count towards the coverage proportion $1-\epsilon$. Therefore, if the class $\mathcal{H}$ from which we choose an $\epsilon$-cover does not have high enough complexity, then the $\epsilon$-covers within $\mathcal{H}$ may need to be exceedingly large in order to achieve $\epsilon$-coverage.
\begin{comment}
% Commenting out the following sentence to make more room in TCNS paper.
As an example, consider covering a line segment in $\mathbb{R}^2$ first with an $\ell_2$-norm ball, and then, instead, with an ellipsoid. Clearly, the additional complexity of the ellipsoid allows for tighter coverage of the line segment.
\end{comment}

The problem with unnecessarily large $\epsilon$-covers is that the feasible set in the optimization defining $\hat{r}(\theta)$ includes many vectors $y$ that may not be actual outputs in $\mathcal{Y}$. In this case, $\hat{r}(\theta)$ is small, even though $\bar{r}(\epsilon)$ may be large. To avoid this problem, our choice of $\Theta$ and $h(\cdot)$ should ensure that the class $\mathcal{H}$ has high enough complexity. However, our choices should also yield a scenario problem \eqref{eq: scenario_problem} that is convex. Indeed, Theorem \ref{thm: convex_scenario_optimization} gives sufficient conditions for the convexity of the scenario problem. Before presenting these conditions, let us recall a fundamental definition for set-valued functions.

\begin{definition}
	\label{def: convexity_of_set-valued_functions}
	A set-valued function $h\colon\mathbb{R}^p\to\mathcal{P}(\mathbb{R}^{n_y})$ is said to be \emph{convex} on a convex set $\Theta\subseteq\mathbb{R}^p$ if
	\begin{equation*}
		\big(\lambda h(\theta_1) + (1-\lambda)h(\theta_2)\big) \subseteq h(\lambda\theta_1 + (1-\lambda)\theta_2)
	\end{equation*}
	for all $\theta_1,\theta_2\in\Theta$ and all $\lambda\in[0,1]$. The function $h(\cdot)$ is said to be \emph{concave} on $\Theta$ if
	\begin{equation*}
		h(\lambda\theta_1 + (1-\lambda)\theta_2) \subseteq \big(\lambda h(\theta_1) + (1-\lambda)h(\theta_2)\big)
	\end{equation*}
	for all $\theta_1,\theta_2\in\Theta$ and all $\lambda\in[0,1]$. Finally, the function $h(\cdot)$ on $\Theta$ is said to be \emph{affine} if it is both convex and concave.
\end{definition}

\begin{comment}
\begin{remark}
	The above definitions of convexity and concavity for a set-valued function are consistent with those used in set-valued optimization and coincide with the traditional definition of cone-convexity. In particular, a convex cone $C\subseteq\mathbb{R}^{n_y}$ defines an order relation on $\mathcal{P}(\mathbb{R}^{n_y})$; $A,B\in\mathcal{P}(\mathbb{R}^{n_y})$ are ordered as $A\le_C B$ if and only if $B\subseteq A+C$ \cite{hamel2015set}. Taking $C=\{0\}$ yields the familiar partial order of subset inclusion, and Definition \ref{def: convexity_of_set-valued_functions} amounts to the usual definition of cone-convexity with respect to the order $\le_{\{0\}}$.
\end{remark}
\end{comment}

\begin{example}
	\label{ex: norm_ball_functions_are_affine}
	Consider the norm ball class $\mathcal{H}$ given in Example \ref{ex: norm_ball_class}. It is easily verified by Definition \ref{def: convexity_of_set-valued_functions} that the set-valued function $h(\cdot)$ defining the class $\mathcal{H}$ is affine on $\Theta=\mathbb{R}^{n_y}\times\mathbb{R}_{++}$.
\end{example}

With tools for defining and proving convexity of set-valued functions now in place, we can present conditions under which the scenario optimization \eqref{eq: scenario_problem} is convex, and therefore easily solvable. In Theorem \ref{thm: high-probability_guarantees}, we will also rely on this convexity to guarantee with high probability that $h(\theta^*)$ is an $\epsilon$-cover of $\mathcal{Y}$ and that our desired lower bound $\hat{r}(\theta^*)\le \bar{r}(\epsilon)$ holds. Generating such guarantees is in general not possible for nonconvex scenario optimization \cite{campi2018general}, further illustrating the importance of Theorem \ref{thm: convex_scenario_optimization} below.

\begin{theorem}
	\label{thm: convex_scenario_optimization}
	Consider the scenario problem \eqref{eq: scenario_problem}. Suppose that $\Theta$ takes the form
	\begin{equation*}
		\Theta = \{\theta\in\mathbb{R}^p : g_i(\theta) \le 0 ~ \text{for all $i\in\{1,2,\dots,m\}$}\},
	\end{equation*}
	where every $g_i\colon\mathbb{R}^p\to\mathbb{R}$ is convex. Furthermore, suppose that $h(\cdot)$ is a concave set-valued function that takes the form
	\begin{equation*}
		h(\theta) = \{y\in\mathbb{R}^{n_y} : h_i(y,\theta) \le 0 ~ \text{for all $i\in\{1,2,\dots,n\}$}\},
	\end{equation*}
	where $h_i\colon\mathbb{R}^{n_y}\times\mathbb{R}^p \to\mathbb{R}$ and $h_i(y,\cdot)$ is convex for all $y\in\mathbb{R}^{n_y}$. Then, \eqref{eq: scenario_problem} is a convex optimization problem.
\end{theorem}

\begin{proof}
	See Appendix \ref{app: proofs}.
\end{proof}

\begin{remark}
	\label{rem: convex_scenario_problem}
	Theorem \ref{thm: convex_scenario_optimization} is easily extended to include affine equality constraints in the forms taken by $\Theta$ and $h(\theta)$. Additionally, if $h_i(y,\theta)$ in Theorem \ref{thm: convex_scenario_optimization} is \emph{jointly} convex in $(y,\theta)$ for all $i$, one can show that $h(\cdot)$ is an affine set-valued function, and therefore $\hat{r}(\cdot)$ in \eqref{eq: scenario_problem} is affine (see, e.g., Proposition 4.2 of \cite{fiacco1986convexity}). Therefore, if $v(\cdot)$ is also affine, the scenario problem \eqref{eq: scenario_problem} has an affine objective.
\end{remark}

Theorem \ref{thm: convex_scenario_optimization} precisely answers our earlier inquiry: the class $\mathcal{H}$ should be complex enough to contain tight $\epsilon$-covers of the output set $\mathcal{Y}$, but at the same time $\Theta$ should be defined by convex constraints and $h(\cdot)$ should be taken as a concave set-valued function also defined by convex constraints. Note that these conditions on $h(\cdot)$ are not as restrictive as they may seem. In particular, Example \ref{ex: norm_ball_functions_are_affine} shows for the norm ball class that $h(\cdot)$ is affine (and therefore concave) and defined by convex constraints, and that this holds \emph{for all norms on $\mathbb{R}^{n_y}$}, even though norm functions themselves are not affine. Therefore, Theorem \ref{thm: convex_scenario_optimization} guarantees that the scenario optimization \eqref{eq: scenario_problem} using the norm ball class is a convex problem. We verify this fact in the following example.

\begin{example}
	\label{ex: scenario_optimization_with_norm_ball_class}
	Recall the norm ball class of Examples \ref{ex: norm_ball_class} and \ref{ex: norm_ball_functions_are_affine}. We show that \eqref{eq: scenario_problem} using this class is convex. Indeed, the approximate robustness level is
	\begin{align*}
		\hat{r}(\bar{y},r) = \inf_{\|y-\bar{y}\|\le r} a^\top y + b = a^\top \bar{y} - r\|a\|_* + b,
	\end{align*}
	which is affine in the optimization variable $\theta=(\bar{y},r)$. Hence, the scenario problem reduces to
	\begin{equation}
		\begin{aligned}
		& \underset{(\bar{y},r)\in\mathbb{R}^{n_y}\times\mathbb{R}_{++}}{\text{maximize}} && b + a^\top \bar{y} - r\|a\|_* - \lambda v(\bar{y},r) \\
		& \text{subject to} && \|y_j-\bar{y}\| \le r ~ \text{for all $j\in\{1,2,\dots,N\}$},
		\end{aligned} \label{eq: scenario_optimization_with_norm_ball_class}
	\end{equation}
	which is a convex problem since $v(\cdot)$ is convex.
\end{example}

% High-Probability Guarantees.

\subsection{High-Probability Guarantees}
\label{sec: high-probability_guarantees}

We now turn to consider the randomness of the scenario problem's optimal value. In particular, we ask the following question: Does the random solution to \eqref{eq: scenario_problem} maintain a certified lower bound on $\bar{r}(\epsilon)$? In Theorem \ref{thm: high-probability_guarantees}, we show that the answer is affirmative with high probability, provided that the problem is convex and a large enough number of samples is used.

\begin{theorem}
	\label{thm: high-probability_guarantees}
	Let $\epsilon,\delta\in[0,1]$. Assume that the scenario optimization \eqref{eq: scenario_problem} is convex and is attained by a solution $\theta^*\in\mathbb{R}^p$. If $N \ge \frac{2}{\epsilon}\left( \log\frac{1}{\delta} + p \right)$, then the following inequalities hold:
	\begin{enumerate}
		\item $\mathbb{P}_{\theta^*}(\mathbb{P}_X(f(X)\in h(\theta^*)) \ge 1-\epsilon) \ge 1-\delta$;
		\item $\mathbb{P}_{\theta^*}(\hat{r}(\theta^*) \le \bar{r}(\epsilon)) \ge 1-\delta$.
	\end{enumerate}
\end{theorem}
\begin{proof}
	See Appendix \ref{app: proofs}.
\end{proof}
The conclusions of Theorem \ref{thm: high-probability_guarantees} assert that, with overwhelming probability, $h(\theta^*)$ is an $\epsilon$-cover of $\mathcal{Y}$ and that the probabilistic robustness level is lower-bounded as $\hat{r}(\theta^*)\le \bar{r}(\epsilon)$. This gives high-probability guarantees for the simultaneous localization and safety certification of the output $Y=f(X)$.

In Theorem \ref{thm: high-probability_guarantees}, randomness of a solution $\theta^*$ to \eqref{eq: scenario_problem} is taken care of by the $1-\delta$ probability bound. In particular, $h(\theta^*)$ may not actually be an $\epsilon$-cover, albeit with probability at most $\delta$. For this reason, we slightly abuse terminology and call $h(\theta^*)$ the optimal $\epsilon$-cover. The additional layer of uncertainty embedded into the parameter $\delta$ is precisely the price paid for replacing the intractable chance-constrained problem \eqref{eq: chance-constrained_problem} with the tractable scenario problem \eqref{eq: scenario_problem}. However, Theorem \ref{thm: high-probability_guarantees} shows that the additional randomness is not an issue, since the requirement on $N$ scales as $\log \frac{1}{\delta}$. Therefore, we can select a small value for $\delta$ while maintaining a reasonable sample size $N$.
\begin{comment}
% Commenting the following two sentences in TCNS submission for more room.
In doing so, the scenario problem can be used in place of the chance-constrained problem to compute the maximum approximate robustness level and lower-bound the probabilistic robustness level of the neural network. The resulting robustness certificate holds with a probability that can be made arbitrarily close to one.
\end{comment}

\begin{remark}
	\label{rem: pac_learning}
	The guarantees in Theorem \ref{thm: high-probability_guarantees} are of the probably approximately correct (PAC) form. In the language of PAC learning, the surrogate output set $h(\theta^*)$ is the hypothesis of the learner, which is selected from the concept class $\mathcal{H} = \{h(\theta) : \theta\in\Theta\}$. Theorem \ref{thm: high-probability_guarantees} asserts that the hypothesis is probably approximately correct, where \emph{approximately correct} means the hypothesis (which is a set) contains the random output $Y=f(X)$ with probability at least $1-\epsilon$, and where \emph{probably} means the hypothesis (which is selected based on the specific instances $x_1,x_2,\dots,x_N$) is approximately correct (for general $X$) with probability at least $1-\delta$. Since this PAC guarantee holds whenever the scenario problem is convex, Theorem \ref{thm: convex_scenario_optimization} gives sufficient conditions for the concept class $\mathcal{H}$ to be PAC learnable, and our proposed method can be viewed as learning robustness using the framework of PAC learning.
\end{remark}

% Exploiting Network Structure.

\section{Exploiting Network Structure}
\label{sec: exploiting_network_structure}
In this section, we show how to exploit the structure of deep neural networks to reduce the time complexity of our method. The basic idea is to utilize adversarial bounds on the deep layers to replace $f$ with a shallower neural network, in effect developing a hybrid adversarial-probabilistic certification scheme. We assume that the network takes the form
	\begin{equation*}
		f = \sigma^{(K)} \circ \mathcal{A}^{(K-1)} \cdots \circ \sigma^{(1)} \circ \mathcal{A}^{(0)},
	\end{equation*}
	where $\sigma^{(k)}\colon\mathbb{R}^{n_k}\to\mathbb{R}^{n_k}$ is the $k^\text{th}$ layer's activation function and $\mathcal{A}^{(k)}\colon\mathbb{R}^{n_k}\to\mathbb{R}^{n_{k+1}}$ is the affine map given by $\mathcal{A}^{(k)}(z) = W^{(k)}z + b^{(k)}$. Note that $n_0 = n_x$ and $n_K = n_y$.

	Now, suppose that $f_L,f_U\colon\mathbb{R}^{n_x}\to\mathbb{R}^{n_y}$ are two functions satisfying
	\begin{equation*}
	f_L(x) \le f(x) \le f_U(x) ~ \text{for all $x\in\mathcal{X}$},
	\end{equation*}
	which are to be determined. Then, define the function $f'\colon\mathbb{R}^{n_x}\to\mathbb{R}^{n_y}$ by
	\begin{equation*}
		f'_i(x) = \begin{aligned}
		\begin{cases}
			(f_L(x))_i & \text{if $a_i \ge 0$}, \\
			(f_U(x))_i & \text{if $a_i < 0$},
		\end{cases}
		\end{aligned}
	\end{equation*}
	for all $i\in\{1,2,\dots,n_y\}$ and all $x\in\mathbb{R}^{n_x}$. It is immediately clear that $a^\top f'(x) + b \le a^\top f(x) + b$ for all $x\in\mathcal{X}$, so $f(x)\in\mathcal{S}$ for all $x\in\mathcal{X}$ such that $f'(x)\in\mathcal{S}$. This shows that
	\begin{equation*}
	\mathbb{P}_X(f'(X) \in\mathcal{S}) \le \mathbb{P}_X(f(X)\in\mathcal{S}).
	\end{equation*}
	Therefore, to certify the probabilistic robustness of $f$, it suffices to apply our certification procedure to the function $f'$. By bounding the deep layers' activations in $f$ by affine functions, we will reduce the problem to analyzing a simpler and shallower network $f'$ that allows for faster sampling of the outputs $y_j$. For notational simplicity, we let $\phi^{(k)} = \sigma^{(k)}\circ \mathcal{A}^{(k-1)} \circ \dots \circ \sigma^{(1)} \circ \mathcal{A}^{(0)}$ for all $k\in\{1,2,\dots,K\}$, so that $\phi^{(k)}(x)$ is the activation at layer $k$ corresponding to the input $x$. Let $\phi^{(0)}$ be the identity map on $\mathbb{R}^{n_x}$. We now recall the notion of preactivation bounds, and make two assumptions.

	\begin{definition}
	\label{def: preactivation_bounds}
	A vector $l^{(k)}\in\mathbb{R}^{n_k}$ satisfying $l^{(k)} \le \mathcal{A}^{(k-1)}\circ \phi^{(k-1)}(x)$ for all $x\in\mathcal{X}$ is called a \emph{$k^\text{th}$ layer preactivation lower bound}. A vector $u^{(k)}\in\mathbb{R}^{n_k}$ satisfying $\mathcal{A}^{(k-1)}\circ \phi^{(k-1)}(x) \le u^{(k)}$ for all $x\in\mathcal{X}$ is called a \emph{$k^\text{th}$ layer preactivation upper bound}.
	\end{definition}

	\begin{assumption}
		\label{ass: preactivation_bounds}
		For all $k\in\{1,2,\dots,K\}$, there exist $k^\text{th}$ layer preactivation lower and upper bounds $l^{(k)}$ and $u^{(k)}$, respectively.
	\end{assumption}

	\begin{assumption}
		\label{ass: affine_activation_bounds}
		For all $k\in\{1,2,\dots,K\}$, there exist functions $\mathcal{L}^{(k)},\mathcal{U}^{(k)}\colon\mathbb{R}^{n_k}\to\mathbb{R}^{n_k}$ given by
	\begin{equation*}
		\mathcal{L}^{(k)}(z) = W^{(k)}_L z + b^{(k)}_L, \quad \mathcal{U}^{(k)}(z) = W^{(k)}_U z + b^{(k)}_U,
	\end{equation*}
	that satisfy $\mathcal{L}^{(k)}(z) \le \sigma^{(k)}(z) \le \mathcal{U}^{(k)}(z)$ for all $z\in [l^{(k)},u^{(k)}]$.
	\end{assumption}

	Definition \ref{def: preactivation_bounds} and Assumptions \ref{ass: preactivation_bounds} and \ref{ass: affine_activation_bounds} are standard in the adversarial robustness literature. Notice that in many common architectures, $n_1>n_0$ and the rank of $\mathcal{A}^{(0)}$ is $n_0$, and in this case Assumption \ref{ass: preactivation_bounds} requires the input set $\mathcal{X}$ to be bounded. For most common activation functions and input sets, there exist a variety of methods for computing the above preactivation bounds and affine bounding functions---see, e.g., \cite{zhang2018efficient}.

	The following lemma transforms our affine bounds on each activation function $\sigma^{(k)}$ into affine bounds relating the activation of one layer to the activation of the next layer.

	\begin{lemma}
		\label{lem: affine_activation_bounds}
		Suppose that Assumptions \ref{ass: preactivation_bounds} and \ref{ass: affine_activation_bounds} hold. For all $k\in\{1,2,\dots,K\}$, it holds for all $x\in\mathcal{X}$ that
		\begin{equation*}
			\tilde{W}_L^{(k)} \phi^{(k-1)}(x) + \tilde{b}_L^{(k)} \le \phi^{(k)}(x) \le \tilde{W}_U^{(k)} \phi^{(k-1)}(x) + \tilde{b}_U^{(k)},
		\end{equation*}
		where
		\begin{equation}
			\begin{aligned}
				\tilde{W}_L^{(k)} = W_L^{(k)} W^{(k-1)}, &\quad \tilde{b}_L^{(k)} = W_L^{(k)} b^{(k-1)} + b_L^{(k)}, \\
				\tilde{W}_U^{(k)} = W_U^{(k)} W^{(k-1)}, &\quad \tilde{b}_U^{(k)} = W_U^{(k)} b^{(k-1)} + b_U^{(k)}.
	\end{aligned} \label{eq: tilde_parameters}
		\end{equation}
	\end{lemma}

	\begin{proof}
		See Appendix \ref{app: proofs}.
	\end{proof}

	Next, we use the affine bounds between each neighboring layer in Lemma \ref{lem: affine_activation_bounds} to develop one overall affine bound relating the activation at some layer $k^*$ to the output $\phi^{(K)}(x)$ of the neural network. The proof technique follows the idea developed in \cite{weng2018towards,zhang2018efficient}, albeit allows for more general activation functions and allows us to ``start'' the affine bounding within the interior of the neural network architecture.

	\begin{proposition}
	\label{prop: replacement}
	Suppose that Assumptions \ref{ass: preactivation_bounds} and \ref{ass: affine_activation_bounds} hold, and assume that $K\ge 3$. Let $k^*\in\{1,2,\dots,K-2\}$ and define $M=K-k^*$. Consider the matrices $\tilde{W}_L^{(k)},\tilde{W}_U^{(k)}$ and vectors $\tilde{b}_L^{(k)},\tilde{b}_U^{(k)}$ defined in \eqref{eq: tilde_parameters}. Define $E_1 = \tilde{W}_L^{(k^*+1)}$, $F_1 = \tilde{b}_L^{(k^*+1)}$, $G_1 = \tilde{W}_U^{(k^*+1)}$, and $H_1 = \tilde{b}_U^{(k^*+1)}$. Also, for $n\in\{2,3,\dots,M\}$, define
		\begin{align*}
			E_n &= \min\{0,\tilde{W}_L^{(k^*+n)}\} G_{n-1} + \max\{0,\tilde{W}_L^{(k^*+n)}\} E_{n-1}, \\
			F_n &= \min\{0,\tilde{W}_L^{(k^*+n)}\} H_{n-1} + \max\{0,\tilde{W}_L^{(k^*+n)}\} F_{n-1} + \tilde{b}_L^{(k^*+n)}, \\
			G_n &= \max\{0,\tilde{W}_U^{(k^*+n)}\} G_{n-1} + \min\{0,\tilde{W}_U^{(k^*+n)}\} E_{n-1}, \\
			H_n &= \max\{0,\tilde{W}_U^{(k^*+n)}\} H_{n-1} + \min\{0,\tilde{W}_U^{(k^*+n)}\} F_{n-1} + \tilde{b}_U^{(k^*+n)}.
		\end{align*}
	Then, for all $x\in\mathcal{X}$, it holds that
	\begin{equation*}
		E_M \phi^{(k^*)}(x) + F_M \le \phi^{(K)}(x) \le G_M \phi^{(k^*)}(x) + H_M.
	\end{equation*}
	\end{proposition}

	\begin{proof}
		See Appendix \ref{app: proofs}.
	\end{proof}

	Since $\phi^{(K)}(x) = f(x)$, Proposition \ref{prop: replacement} shows that we may take the functions $f_L,f_U$ to be $f_L=\mathcal{A}_L\circ \phi^{(k^*)}$ and $f_U=\mathcal{A}_U\circ \phi^{(k^*)}$, where $\mathcal{A}_L(z) = E_M z + F_M$ and $\mathcal{A}_U(z) = G_M z + H_M$. In this case, our function $f'$ becomes
	\begin{equation*}
		f'_i(x) = \begin{aligned}
		\begin{cases}
		\left(\mathcal{A}_L \circ \phi^{(k^*)}(x) \right)_i & \text{if $a_i \ge 0$}, \\
		\left( \mathcal{A}_U \circ \phi^{(k^*)}(x) \right)_i & \text{if $a_i < 0$}.
		\end{cases}
		\end{aligned}
	\end{equation*}
	This function $f'$ is a new neural network with the same first $k^*<K$ nonlinear layers as $f$, and with one final affine transformation. Thus, a lower bound on the probabilistic robustness level of this shallow surrogate network $f'$ is also a lower bound on the probabilistic robustness level of the deep original network $f$. 

	When $k^*$ is chosen to be small, the depth of this surrogate network is reduced, making it more efficient to sample outputs from it. As $k^*$ increases, our method incorporates more of the underlying nonlinear nature of the network $f$ into the samples that we use to assess $f$'s robustness, meaning that the robustness certificate becomes tighter, but at the expense of increased sampling time. Specifically, in the common setting where every activation $\sigma^{(k)}$ is an element-wise operator with the time complexity $O(n_k)$, the time complexity of the sampling procedure for $f$ is $O(N(n_0n_1+n_1n_2+\cdots+n_{K-1}n_K))$, whereas the time complexity for $f'$ is $O(N(n_0n_1+n_1n_2+\cdots+n_{k^*-1}n_{k^*} + n_{k^*}n_K))$. If, for example, every number $n_k$ is of order $O(n)$, then $f$ would have the sampling time complexity $O(NKn^2)$, whereas $f'$ would be of order $O(Nk^* n^2)$, giving a factor of $k^*/K$ reduction in time complexity. As we will see in Section \ref{sec: exploiting_network_structure_experiment}, this reduced time complexity is particularly helpful in deep neural network settings.

% Numerical Experiments.

\section{Numerical Experiments}
\label{sec: numerical_experiments}

\subsection{Illustrative Example}
\label{sec: illustrative_example}
	We consider the distributed linear system $x(t+1) = Ax(t) + Bu(t)$ for times $t\in\{0,1,\dots,T\}$, $T=20$, as constructed in \cite{gama2021graph}. The system has $n=10$ nodes, with a single state and input associated with every node; $x(t),u(t)\in\mathbb{R}^n$. The system and control matrices $A,B$ respect the underlying graph topology of the system, encoded by the support matrix $S$---see \cite{gama2021graph}.

	The control law is defined by a graph neural network:
	\begin{equation*}
	u(t)=\Phi(x(t),S) \coloneqq \sum_{k=0}^{K-1} h^{(2)}_{k+1} S^k \sigma\left( \sum_{j=0}^{J-1} h^{(1)}_{j+1}S^j x(t) \right),
	\end{equation*}
	with $\sigma(\cdot) = \relu(\cdot)$, $K=J=3$, $h^{(1)} = (\frac{1}{2},\frac{1}{2},\frac{1}{2})$, and $h^{(2)}=(1,1,1)$. This neural network controller, defined in terms of $S$, respects the distributed nature of the system \cite{gama2021graph}. In this experiment, we consider the case where the graph support of the control law may be randomly perturbed, so that $u(t) = \Phi(x(t),S')$ for some $S'\in\mathbb{R}^{n\times n}$ with $S'_{ij} = X S_{ij}$, where $X$ is a Bernoulli random variable equal to $1$ with probability $0.8$; the controller loses an edge in its support graph with probability $0.2$. We fix a (normal random) initial condition $x(0)\in\mathbb{R}^n$, and we consider the map $f\colon\mathbb{R}^{n\times n}\to\mathbb{R}^2$ given by $f(S') = (x_1(T),x_2(T))$, where $x(T)$ is the terminal state of the system under the control law given by $u(t)=\Phi(x(t),S')$. The safe set is defined by $\mathcal{S}_1 = \{y\in\mathbb{R}^2 : a^\top y + b\ge 0\}$, where $a=(1,0)$ and $b=0.05$. We seek to certify that the first two elements of the (random) terminal state are safe even under the perturbed control support $S'$, i.e., that $f(S')\in\mathcal{S}_1$.

	The norm ball class $\mathcal{H}$ of Examples \ref{ex: norm_ball_class}, \ref{ex: norm_ball_functions_are_affine}, and \ref{ex: scenario_optimization_with_norm_ball_class} is employed with $\|\cdot\|$ being the $\ell_2$-norm, and with probability levels $\epsilon=0.05$ and $\delta=10^{-5}$. We set $N = \left\lceil \frac{2}{\epsilon}(\log\frac{1}{\delta}+p) \right\rceil = 581$, then sample $N$ inputs $S_j'$ and compute their corresponding outputs $f(S_j')$ by running the system. As shown in Example \ref{ex: norm_ball_functions_are_affine}, $h(\cdot)$ is an affine set-valued function, and therefore $\Theta$ and $h(\cdot)$ satisfy the conditions of Theorem \ref{thm: convex_scenario_optimization}. We choose the regularizer for the scenario problem \eqref{eq: scenario_optimization_with_norm_ball_class} to be the square of the norm ball radius, i.e., $v(\bar{y},r)=r^2$. The optimization problem is convex as guaranteed by Theorem \ref{thm: convex_scenario_optimization}. We solve the scenario problem first without regularization, and then with two different levels of regularization: $\lambda_1=1$ and $\lambda_2=100$. The respective solutions are denoted by $\theta^*$, $\theta^*_{\lambda_1}$, and $\theta^*_{\lambda_2}$. Each instance takes approximately $15$ seconds to solve using CVX in \textsc{Matlab} on a standard laptop with a $\SI{2.6}{\giga\hertz}$ dual-core i5 processor. The resulting approximate robustness levels are $\hat{r}(\theta^*) = 0.0058$, $\hat{r}(\theta^*_{\lambda_1}) = 0.0054$, and $\hat{r}(\theta^*_{\lambda_2}) = -0.0061$. In the instances without regularization and with regularization level $\lambda_1$, Theorem \ref{thm: high-probability_guarantees} guarantees that the perturbed terminal state $(x_1(T),x_2(T))$ has a safety level of $0.005$ with our prescribed high probability, granting the probablistic robustness certificate we seek. On the other hand, since $\hat{r}(\theta_{\lambda_2}^*)<0$, the scenario problem using regularization level $\lambda_2$ is not able to certify the safety of the terminal state. This is due to the inherent tradeoff between localization and certification, which we now discuss.

The optimal $\epsilon$-covers $h(\theta^*)$, $h(\theta^*_{\lambda_1})$, and $h(\theta^*_{\lambda_2})$ are shown in Figure \ref{fig: control_one_halfspace}. The unregularized set $h(\theta^*)$ is massively over-conservative due to the choice $\lambda=0$, which corresponds to pure robustness certification. Indeed, $h(\theta^*)$ is the $\epsilon$-cover from our class of sets that is furthest from the boundary of the safe set, making $\hat{r}(\theta^*)$ the tightest lower bound on $\bar{r}(\epsilon)$. On the other hand, the optimal $\epsilon$-covers using $\lambda=\lambda_1$ and $\lambda=\lambda_2$ are seen to give tighter localizations of the terminal state $(x_1(T),x_2(T))$. The approximate robustness level using regularization $\lambda_1$ is only slightly lower than the unregularized value, but the regularization $\lambda_2$ is large enough to cause the approximate robustness level $\hat{r}(\theta^*_{\lambda_2})$ to become negative at the expense of localization. This shows how overemphasizing localization may harm the certification aspect of robustness assessment, and empirically demonstrates why output set estimation methods may not be adequate for issuing robustness certificates. This is explored further in Section \ref{sec: comparison_to_output_set_estimation}.

\begin{figure}[ht]
	\centering
	\includegraphics[width=0.6\linewidth]{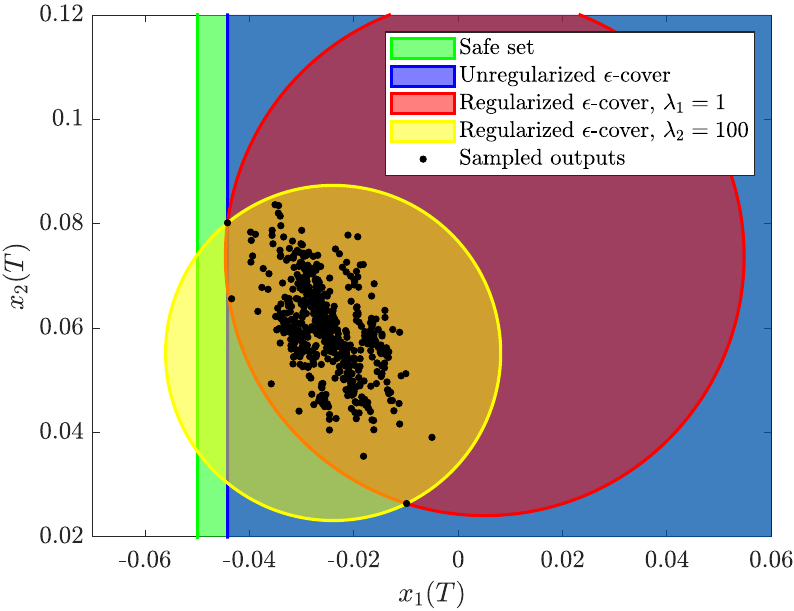}
	\caption{Optimal $\ell_2$-norm ball $\epsilon$-covers for safe set $\mathcal{S}_1$.}
	\label{fig: control_one_halfspace}
\end{figure}

We repeat the experiment with the more complicated safe set $\mathcal{S}_2 = \{y\in\mathbb{R}^2 : Ay + b \ge 0\}$, where $A = \left[\begin{smallmatrix}
	1 & 0 \\
	-1 & 0
\end{smallmatrix}\right]$ and $b = (0.05, 0)$, applying our method to each row of $\mathcal{S}_2$ individually. To do so, we set $\epsilon'=\epsilon/2$, $\delta'=\delta/2$, and $N' = \lceil \frac{2}{\epsilon'}\left(\log\frac{1}{\delta'}+p\right)\rceil = 1217$. For each of the two half-spaces defining $\mathcal{S}_2$, we solve the scenario problem using $N'$ independent and identically distributed samples, and then intersect the two resulting $\epsilon'$-covers. Doing so, we obtain an $\epsilon$-cover with probability at least $1-\delta$. We repeat this process again using regularization levels $\lambda_1=1$ and $\lambda_2=100$, and we find that each scenario problem takes approximately $30$ seconds to solve. As seen in Figure \ref{fig: control_two_halfspaces}, some samples may reside outside the resulting intersection $\epsilon$-covers---this is valid, and the robustness certificates still hold.

Again, we find robustness certificates for $\lambda=0$ and $\lambda=\lambda_1$. However, for $\lambda=\lambda_2$, the optimal $\epsilon$-covers corresponding to both half-spaces are found to intersect the unsafe region of the state space, due to the increased emphasis on localization. Interestingly, the overall localization after intersecting the two $\epsilon'$-covers for $\lambda=\lambda_2$ is in a sense looser than that of the case $\lambda=\lambda_1$, indicating that moderate regularization levels, like $\lambda_1$ in this experiment, may simultaneously perform best for both localization and certification in the case of safe sets defined by more than one half-space. Optimizing $\lambda$ in general poses an interesting problem for future research.

\begin{figure}[ht]
	\centering
	\includegraphics[width=0.6\linewidth]{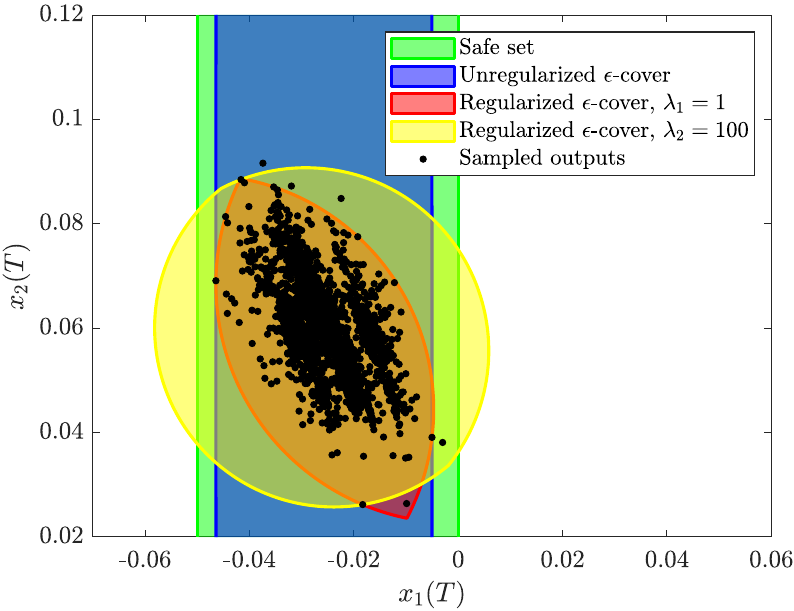}
	\caption{Optimal $\epsilon$-covers amongst intersections of two $\ell_2$-norm ball $\frac{\epsilon}{2}$-covers for safe set $\mathcal{S}_2$.}
	\label{fig: control_two_halfspaces}
\end{figure}

\subsection{Comparison to Output Set Estimation}
\label{sec: comparison_to_output_set_estimation}

In this example, we compare our proposed method to an alternate approach. In the second approach, we first estimate the output set of the neural network using the scenario-based reachability analysis in \cite{devonport2020estimating}. We then use the resulting output set estimate to assess robustness. Recall that our proposed scenario optimization \eqref{eq: scenario_problem} generalizes the reachability analysis of \cite{devonport2020estimating}. In addition to localizing the network outputs, our approach directly takes the goal of robustness certification into account, whereas the estimation technique of \cite{devonport2020estimating} does not.

To illustrate our comparison, consider a simple ReLU neural network given by $f\colon\mathbb{R}^2\to\mathbb{R}^2$, where $f_i(x) = \max\{0,x_i\}$ for $i\in\{1,2\}$. The noisy input $X$ is distributed uniformly on the input set $\mathcal{X} = \{x\in\mathbb{R}^2 : \|x-\bar{x}\|_1\le 1\}$, where $\bar{x}=(1,0)$. The safe set is given as $\mathcal{S} = \{y\in\mathbb{R}^2 : a^\top y + b \ge 0\}$, where $a=(0,1)$ and $b=0.5$. It is straightforward to show that the output set is the top-half of the input set, namely, $\mathcal{Y}=\mathcal{X}\cap\{y\in\mathbb{R}^2 : y_2\ge 0\}$. Hence, if $y\in\mathcal{Y}$ then $a^\top y + b = y_2+b\ge b\ge 0$. Therefore, $\mathcal{Y}\subseteq\mathcal{S}$, and so the random output $Y=f(X)$ is safe with probability one.
\begin{comment}
% Commenting below sentence to make room in TCNS paper.
The network is deterministically robust (and therefore has nonnegative probabilistic robustness level as well).
\end{comment}

We now perform the two assessments at hand, computing our proposed solution first. We choose the $\ell_2$-norm ball class for our candidate $\epsilon$-covers and draw sufficiently many output samples $\{y_j\}_{j=1}^N$ according to Theorem \ref{thm: high-probability_guarantees} with $\epsilon=0.1$ and $\delta=10^{-5}$. Next, we choose the regularizer $v(\bar{y},r)=r^2$ with $\lambda=0.1$ and solve the scenario problem \eqref{eq: scenario_optimization_with_norm_ball_class} for the $\ell_2$-norm ball class. The solution correctly certifies that network outputs are safe with high probability; see the blue set in Figure \ref{fig: devonport_comparison}.

We now turn to the alternative method. We use the same $\ell_2$-norm ball class as above and solve for the minimum volume $\epsilon$-cover using the same $N$ sampled outputs. The estimated output set is shown in red in Figure \ref{fig: devonport_comparison}. Despite being a tighter localization, a substantial portion of the estimated output set exits the safe set, meaning that this approach cannot certify the robustness of the network, even though the random output is truly safe with probability one. This comparison shows that a good estimate of the output set may not be the most informative set to use for assessing output safety. This observation endorses our proposed method, which simultaneously encodes both goals of certification and localization.

\begin{figure}[ht]
	\centering
	\includegraphics[width=0.6\linewidth]{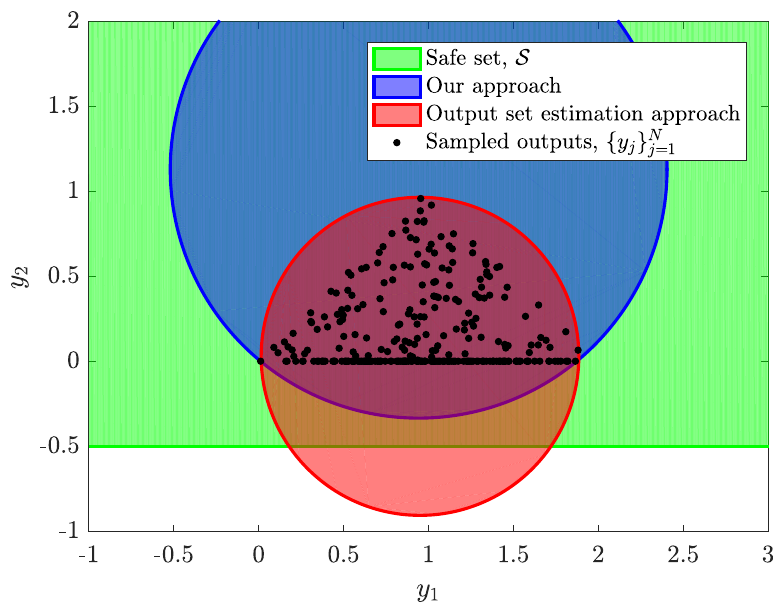}
	\caption{The tightest $\epsilon$-cover of the output set (red) does not correctly certify robustness. Our approach (blue) correctly certifies robustness and maintains reasonable localization.}
	\label{fig: devonport_comparison}
\end{figure}

\subsection{Comparison to PROVEN}
\label{sec: comparison_to_proven}

\begin{table*}[ht]
	\centering
	\caption{Average probabilistic robustness level lower bounds $\hat{r}(\theta^*)$ for MNIST ReLU networks subject to uniform noise over $\ell_\infty$-norm ball. All values are averaged over $10$ nominal inputs with randomly chosen target classes $i$. Lower bounds giving certified robustness (on average) are bolded, and the average certified adversarial radii computed using \cite{zhang2018efficient} are italicized.}
	\begin{subtable}{0.45\textwidth}
	\centering
	\caption{$2\times[20]$ network.}
	\resizebox{\textwidth}{!}{%
	\begin{tabular}{l r r r r r r}
	\toprule
	\multirow{2}{*}{Radius} & \multicolumn{2}{c}{$\epsilon=0.001$} & \multicolumn{2}{c}{$\epsilon=0.1$} & \multicolumn{2}{c}{$\epsilon=0.25$} \\
	& PRVN & Ours & PRVN & Ours & PRVN & Ours \\
	\midrule %
	\multirow{2}{*}{$0.01$} & $\bf 24.51$ & $\bf 14.11$ & $\bf 24.79$ & $\bf 14.26$ & $\bf 24.88$ & $\bf 14.28$ \\
	& $\SI{1.491}{\second}$ & $\SI{0.605}{\second}$ & $\SI{1.405}{\second}$ & $\SI{0.004}{\second}$ & $\SI{1.370}{\second}$ & $\SI{0.003}{\second}$ \\%
\multirow{2}{*}{$\mathit{0.027}$} & $\bf 14.71$ & $\bf 13.36$ & $\bf 15.45$ & $\bf 13.77$ & $\bf 15.68$ & $\bf 13.85$ \\
	& $\SI{1.434}{\second}$ & $\SI{0.599}{\second}$ & $\SI{1.540}{\second}$ & $\SI{0.004}{\second}$ & $\SI{1.427}{\second}$ & $\SI{0.003}{\second}$ \\%
	\multirow{2}{*}{$0.05$} & $-1.33$ & $\bf 12.34$ & $\bf 0.02$ & $\bf 13.11$ & $\bf 0.44$ & $\bf 13.26$ \\
	& $\SI{1.511}{\second}$ & $\SI{0.597}{\second}$ & $\SI{1.468}{\second}$ & $\SI{0.004}{\second}$ & $\SI{1.423}{\second}$ & $\SI{0.003}{\second}$ \\%
	\multirow{2}{*}{$0.1$} & $-42.09$ & $\bf 10.25$ & $-39.43$ & $\bf 11.66$ & $-38.61$ & $\bf 12.00$ \\
	& $\SI{1.485}{\second}$ & $\SI{0.623}{\second}$ & $\SI{1.437}{\second}$ & $\SI{0.004}{\second}$ & $\SI{1.437}{\second}$ & $\SI{0.002}{\second}$ \\%
	\multirow{2}{*}{$0.5$} & $-404.42$ & $-7.21$ & $-391.05$ & $-0.05$ & $-386.96$ & $\bf 1.71$ \\
	& $\SI{1.525}{\second}$ & $\SI{0.645}{\second}$ & $\SI{1.472}{\second}$ & $\SI{0.004}{\second}$ & $\SI{1.432}{\second}$ & $\SI{0.002}{\second}$ \\%
	\bottomrule
	\end{tabular}%
	}
	\label{tab: mnist_relu_2x20}
	\end{subtable}
	\hfil
	\begin{subtable}{0.45\textwidth}
	\centering
	\caption{$3\times[20]$ network.}
	\resizebox{\textwidth}{!}{%
	\begin{tabular}{l r r r r r r}
	\toprule
	\multirow{2}{*}{Radius} & \multicolumn{2}{c}{$\epsilon=0.001$} & \multicolumn{2}{c}{$\epsilon=0.1$} & \multicolumn{2}{c}{$\epsilon=0.25$} \\
	& PRVN & Ours & PRVN & Ours & PRVN & Ours \\
	\midrule %
	\multirow{2}{*}{$0.01$} & $\bf 29.24$ & $\bf 17.28$ & $\bf 29.59$ & $\bf 17.45$ & $\bf 29.70$ & $\bf 17.49$ \\
	& $\SI{1.416}{\second}$ & $\SI{0.380}{\second}$ & $\SI{1.345}{\second}$ & $\SI{0.003}{\second}$ & $\SI{1.388}{\second}$ & $\SI{0.001}{\second}$ \\%
\multirow{2}{*}{$\mathit{0.022}$} & $\bf 18.80$ & $\bf 16.65$ & $\bf 19.49$ & $\bf 17.02$ & $\bf 19.71$ & $\bf 17.10$ \\
	& $\SI{1.382}{\second}$ & $\SI{0.362}{\second}$ & $\SI{1.345}{\second}$ & $\SI{0.003}{\second}$ & $\SI{1.364}{\second}$ & $\SI{0.001}{\second}$ \\%
	\multirow{2}{*}{$0.05$} & $-22.31$ & $\bf 15.19$ & $-20.67$ & $\bf 16.00$ & $-20.17$ & $\bf 16.19$ \\
	& $\SI{1.377}{\second}$ & $\SI{0.345}{\second}$ & $\SI{1.325}{\second}$ & $\SI{0.003}{\second}$ & $\SI{1.374}{\second}$ & $\SI{0.001}{\second}$ \\%
	\multirow{2}{*}{$0.1$} & $-114.83$ & $\bf 12.57$ & $-111.59$ & $\bf 14.19$ & $-110.60$ & $\bf 14.58$ \\
	& $\SI{1.351}{\second}$ & $\SI{0.372}{\second}$ & $\SI{1.343}{\second}$ & $\SI{0.003}{\second}$ & $\SI{1.340}{\second}$ & $\SI{0.002}{\second}$ \\%
	\multirow{2}{*}{$0.5$} & $-866.28$ & $-9.36$ & $-857.82$ & $-0.55$ & $-855.22$ & $\bf 0.36$ \\
	& $\SI{1.385}{\second}$ & $\SI{0.368}{\second}$ & $\SI{1.351}{\second}$ & $\SI{0.003}{\second}$ & $\SI{1.336}{\second}$ & $\SI{0.003}{\second}$ \\%
	\bottomrule
	\end{tabular}%
	}
	\label{tab: mnist_relu_3x20}
	\end{subtable} \\
	\vspace*{\baselineskip}%
	\begin{subtable}{0.45\textwidth}
	\centering
	\caption{$2\times[1024]$ network.}
	\resizebox{\textwidth}{!}{%
	\begin{tabular}{l r r r r r r}
	\toprule
	\multirow{2}{*}{Radius} & \multicolumn{2}{c}{$\epsilon=0.001$} & \multicolumn{2}{c}{$\epsilon=0.1$} & \multicolumn{2}{c}{$\epsilon=0.25$} \\
	& PRVN & Ours & PRVN & Ours & PRVN & Ours \\
	\midrule %
	\multirow{2}{*}{$0.01$} & $\bf 51.07$ & $\bf 27.73$ & $\bf 51.40$ & $\bf 27.87$ & $\bf 51.50$ & $\bf 27.93$ \\
	& $\SI{0.899}{\second}$ & $\SI{1.102}{\second}$ & $\SI{0.877}{\second}$ & $\SI{0.008}{\second}$ & $\SI{0.874}{\second}$ & $\SI{0.004}{\second}$ \\%
\multirow{2}{*}{$\mathit{0.032}$} & $\bf 30.34$ & $\bf 26.69$ & $\bf 31.37$ & $\bf 27.13$ & $\bf 31.69$ & $\bf 27.32$ \\
	& $\SI{0.869}{\second}$ & $\SI{1.202}{\second}$ & $\SI{0.858}{\second}$ & $\SI{0.008}{\second}$ & $\SI{0.846}{\second}$ & $\SI{0.004}{\second}$ \\%
	\multirow{2}{*}{$0.05$} & $\bf 6.95$ & $\bf 25.83$ & $\bf 8.59$ & $\bf 26.53$ & $\bf 9.09$ & $\bf 26.82$ \\
	& $\SI{0.846}{\second}$ & $\SI{1.125}{\second}$ & $\SI{0.851}{\second}$ & $\SI{0.008}{\second}$ & $\SI{0.844}{\second}$ & $\SI{0.004}{\second}$ \\%
	\multirow{2}{*}{$0.1$} & $-77.53$ & $\bf 23.46$ & $-74.13$ & $\bf 24.84$ & $-73.09$ & $\bf 25.42$ \\
	& $\SI{0.861}{\second}$ & $\SI{1.137}{\second}$ & $\SI{0.854}{\second}$ & $\SI{0.008}{\second}$ & $\SI{0.881}{\second}$ & $\SI{0.004}{\second}$ \\%
	\multirow{2}{*}{$0.5$} & $-914.36$ & $\bf 4.83$ & $-900.22$ & $\bf 11.79$ & $-895.89$ & $\bf 14.45$ \\
	& $\SI{0.869}{\second}$ & $\SI{1.159}{\second}$ & $\SI{0.883}{\second}$ & $\SI{0.008}{\second}$ & $\SI{0.868}{\second}$ & $\SI{0.004}{\second}$ \\%
	\bottomrule
	\end{tabular}%
	}
	\label{tab: mnist_relu_2x1024}
	\end{subtable}
	\hfil
	\begin{subtable}{0.45\textwidth}
	\centering
	\caption{$3\times[1024]$ network.}
	\resizebox{\textwidth}{!}{%
	\begin{tabular}{l r r r r r r}
	\toprule
	\multirow{2}{*}{Radius} & \multicolumn{2}{c}{$\epsilon=0.001$} & \multicolumn{2}{c}{$\epsilon=0.1$} & \multicolumn{2}{c}{$\epsilon=0.25$} \\
	& PRVN & Ours & PRVN & Ours & PRVN & Ours \\
	\midrule %
	\multirow{2}{*}{$0.01$} & $\bf 68.87$ & $\bf 36.86$ & $\bf 69.28$ & $\bf 37.06$ & $\bf 69.41$ & $\bf 37.12$ \\
	& $\SI{2.535}{\second}$ & $\SI{1.782}{\second}$ & $\SI{2.382}{\second}$ & $\SI{0.015}{\second}$ & $\SI{2.464}{\second}$ & $\SI{0.009}{\second}$ \\%
\multirow{2}{*}{$\mathit{0.024}$} & $\bf 44.14$ & $\bf 35.97$ & $\bf 45.18$ & $\bf 36.44$ & $\bf 45.50$ & $\bf 36.58$ \\
	& $\SI{2.434}{\second}$ & $\SI{2.026}{\second}$ & $\SI{2.448}{\second}$ & $\SI{0.013}{\second}$ & $\SI{2.510}{\second}$ & $\SI{0.008}{\second}$ \\%
	\multirow{2}{*}{$0.05$} & $-111.09$ & $\bf 34.32$ & $-108.25$ & $\bf 35.32$ & $-107.39$ & $\bf 35.59$ \\
	& $\SI{2.739}{\second}$ & $\SI{2.258}{\second}$ & $\SI{2.671}{\second}$ & $\SI{0.013}{\second}$ & $\SI{2.761}{\second}$ & $\SI{0.007}{\second}$ \\%
	\multirow{2}{*}{$0.1$} & $-729.24$ & $\bf 31.10$ & $-723.45$ & $\bf 33.10$ & $-721.68$ & $\bf 33.69$ \\
	& $\SI{3.081}{\second}$ & $\SI{2.325}{\second}$ & $\SI{2.916}{\second}$ & $\SI{0.014}{\second}$ & $\SI{2.912}{\second}$ & $\SI{0.007}{\second}$ \\%
	\multirow{2}{*}{$0.5$} & $-6872.3$ & $\bf 6.89$ & $-6849.5$ & $\bf 15.85$ & $-6842.5$ & $\bf 18.56$ \\
	& $\SI{2.877}{\second}$ & $\SI{1.955}{\second}$ & $\SI{2.996}{\second}$ & $\SI{0.014}{\second}$ & $\SI{3.012}{\second}$ & $\SI{0.007}{\second}$ \\%
	\bottomrule
	\end{tabular}%
	}
	\label{tab: mnist_relu_3x1024}
	\end{subtable}
	\label{tab: compare_to_proven}
\end{table*}

\begin{table*}[ht]
	\centering
	\caption{Average probabilistic robustness level lower bounds $\hat{r}(\theta^*)$ for various other models. Values for Models $1$ and $2$ are averaged over $10$ inputs, and for Model $3$ they are averaged over $100$ network realizations. Lower bounds giving certified robustness (on average) are bolded, and the average certified adversarial radii computed using \cite{zhang2018efficient} are italicized.}
	\begin{tabular}{l r r l r r l r r}
	\toprule
	\multicolumn{3}{c}{Model $1$} & \multicolumn{3}{c}{Model $2$} & \multicolumn{3}{c}{Model $3$} \\
	Radius & PRVN & Ours & Radius & PRVN & Ours & Radius &PRVN & Ours \\
	\midrule %
	$0.005$ & $\bf 7.81$ & $\bf 19.01$ & $0.001$ & $\bf 48.80$ & $\bf 33.26$ & $0.01$ & $\bf 1.96$ & $\bf 1.97$ \\
	$\mathit{0.0068}$ & $\bf 2.20$ & $\bf 18.96$ & $\mathit{0.0023}$ & $\bf 10.90$ & $\bf 33.17$ & $0.05$ & $\bf 1.71$ & $\bf 1.79$ \\
	$0.01$ & $-26.31$ & $\bf 18.88$ & $0.003$ & $-91.20$ & $\bf 33.12$ & $0.1$ & $\bf 1.40$ & $\bf 1.56$ \\
	$0.05$ & $-1769.97$ & $\bf 17.83$ & $0.005$ & $-1056.85$ & $\bf 32.98$ & $0.5$ & $-1.06$ & $-0.29$ \\
	$0.1$ & $-4493.02$ & $\bf 16.48$ & $0.01$ & $-8717.05$ & $\bf 32.62$ & $1.0$ & $-4.13$ & $-2.60$ \\
	\bottomrule
	\end{tabular}
	\label{tab: other_models}
\end{table*}

In this experiment, we compare our approach using the half-space class $\mathcal{H} = \big\{\{y\in\mathbb{R}^{n_y} : c^\top y + d \ge 0\} : (c,d)\in\mathbb{R}^{n_y}\times\mathbb{R}\big\}$, for which we solve the scenario problem using its closed-form solution (see Appendix \ref{sec: special_case-class_of_half-spaces}), to the state-of-the-art algorithm, PROVEN \cite{weng2019proven}, for assessing robustness against random input noise. Throughout, we use open-source neural network models provided in \cite{weng2019proven}. The underlying framework of PROVEN relies on bounding a classifier's margin function by affine functions. PROVEN uses the affine functions to give closed-form bounds on the misclassification probability. We remark that, since PROVEN does not rely on sampling, their lower bound on $\bar{r}(\epsilon)$ is deterministic, whereas our bound holds with probability $1-\delta$, which is taken to be $1-10^{-5} = 0.99999$ in this experiment. The results in this section are computed using TensorFlow in Python on a standard laptop with a $\SI{2.6}{\giga\hertz}$ dual-core i5 processor.

We first consider a variety of pre-trained MNIST digit classification networks with ReLU activation functions \cite{lecun1998mnist}. A network model with $m$ hidden layers, each having $n$ neurons, is denoted by $m\times[n]$. We model the noisy input $X$ as being distributed uniformly on $\mathcal{X}=\{x\in\mathbb{R}^{n_x} : \|x-\bar{x}\|_\infty \le \epsilon_x\}$. For $10$ randomly selected nominal inputs $\bar{x}$, we compute a lower bound $\hat{r}(\theta^*)$ on the probabilistic robustness level $\bar{r}(\epsilon)$. The robustness level of a network (for a particular pair $(\epsilon,\epsilon_x)$) is evaluated by computing the average robustness level lower bound across the $10$ inputs.\footnote{Despite $\bar{r}(\epsilon)$ being an input-specific quantity, we follow the literature's standard practice and average our robustness metric over a collection of test inputs. This standard was popularized in \cite{szegedy2014intriguing}, where model robustness is evaluated using average certified input set radii. Our average robustness level lower bound immediately gives an average certified input set radius when the bound is nonnegative. In the probabilistic setting, it can be more natural to evaluate models in terms of misclassification probability, like our bounds do, instead of in terms of certified input set radii, see, e.g., \cite{zakrzewski2004randomized,webb2019statistical,fazlyab2019probabilistic,couellan2021probabilistic}.} This is done for probability levels $\epsilon\in\{0.001,0.1,0.25\}$ (with corresponding sample sizes $N\in\{25026,251,101\}$) and for a variety of noise levels $\epsilon_x$. We include the certified adversarial radius computed using \cite{zhang2018efficient}, which is a lower bound on the smallest radius such that $\mathcal{X}$ contains an input that yields an unsafe output. The targeted class $i$, which defines the margin function $g_i$ relative to the nominal input's true class $i^*$, is randomly chosen for each input tested. See Examples \ref{ex: classification_network_setup} and \ref{ex: classification_network_probabilistic_robustness_level} for more information on this application. The average lower bound values computed using our approach (denoted Ours) and PROVEN's (denoted PRVN) are shown in Table \ref{tab: compare_to_proven}.

As seen in Table \ref{tab: compare_to_proven}, our method is able to certify larger input sets than PROVEN for every network tested. Although PROVEN's lower bound is tighter for small radii, at large radii our bound is significantly tighter than PROVEN's, particularly for the larger networks in Tables \ref{tab: mnist_relu_2x1024} and \ref{tab: mnist_relu_3x1024}. This indicates that our method is especially powerful for certifying deep neural networks. The end-to-end affine bounding scheme in PROVEN tends to become looser as the network becomes deeper and as the input set becomes larger \cite{weng2019proven}. The technique comes from the adversarial robustness literature, and therefore it being embedded into PROVEN is likely the reason why PROVEN fails for radii larger than the certified adversarial radius. Our method bypasses this preliminary bound altogether. We also remark that our method certifies much larger input set radii (sometimes up to $20$ times larger) compared to the certified adversarial radii (italicized) computed using the state-of-the-art worst-case analysis \cite{zhang2018efficient}. The exact minimum adversarial radii (averaged across the $10$ inputs) for the $2\times[20]$ and $3\times[20]$ ReLU networks are efficiently computed to be around $0.07$ using mixed-integer linear programming \cite{tjeng2019evaluating}. With the tolerance $\epsilon=0.001$, our method certifies radii over $0.1$ for these networks. This evidences the claim that worst-case approaches, including exact ones, are over-conservative when applied to settings where a small amount of risk may be tolerable, in effect justifying our data-driven framework.

In Table \ref{tab: other_models}, we repeat the experiment using three variants in the neural network model. Model $1$ is an MNIST classifier with $\tanh(\cdot)$ activation functions of size $4\times[1024]$. On the other hand, Model $2$ is a CIFAR-10 network with ReLU activations of size $5\times [2048]$. We see that both Model $1$ and Model $2$ exhibit the same behavior as before; for small input set radii, the lower bounds provided by PROVEN and our method are similar and both yield high-probability robustness certificates. For larger radii, our lower bound significantly outperforms PROVEN's. Since the affine bounds in PROVEN are relatively tight for small input sets radii, we suspect the PROVEN bound to closer match our bound for large input set radii in the special case  of linear classification networks.

Model $3$ is a linear classifier, i.e., of the form $f(x)=Wx+b$, with $50$ inputs, $10$ outputs, and weights, biases, and nominal input all chosen randomly with elements uniform on $[0,1]$. We computed lower bounds on $\bar{r}(\epsilon)$ for $100$ such models and averaged the results. Table \ref{tab: other_models} shows that indeed the PROVEN bound closely matches our bound for every radius tested in this special case, and that the two methods succeed and fail to issue robustness certificates simultaneously. These results show that the worst-case bounding techniques used in the adversarial robustness literature may work satisfactorily for simple models with random inputs, such as linear classifiers, but that these bounds are too loose for general nonlinear networks.

\subsection{Exploiting Network Structure}
\label{sec: exploiting_network_structure_experiment}

	In this experiment, we implement the complexity-reducing method of Section \ref{sec: exploiting_network_structure}. We consider networks with $10$ inputs, $10$ outputs, and $250$ neurons in every hidden layer. The number of layers $K$ varies from $3$ to $25$. The weights and biases for every architecture are chosen randomly (with Gaussian elements, then normalized). Every activation function $\sigma^{(k)}$ is chosen to be ReLU, with preactivation and affine bounds derived according to \cite{weng2018towards}. We consider (randomly chosen Gaussian) clean inputs $\bar{x}$ with uniform additive random noise on the $\ell_\infty$-norm ball with radius $\epsilon_x = 0.1$, so that the noisy inputs $X$ are distributed uniformly on $\{x\in\mathbb{R}^{n_x} : \|x-\bar{x}\|_\infty\le \epsilon_x\}$.

	For every architecture, we lower-bound the probabilistic robustness level for $50$ different realizations of the weights, biases, and inputs, where for each realization we solve the scenario optimization problem using the class $\mathcal{H} = \big\{\{y\in\mathbb{R}^{n_y} : c^\top y + d \ge 0\} : (c,d)\in\mathbb{R}^{n_y}\times\mathbb{R}\big\}$ of half-spaces with $N=1000$ sampled inputs. This is done both using our baseline methodology, maintaining the full nonlinearity of each deep network, as well as using the shallow surrogate networks proposed in Section \ref{sec: exploiting_network_structure}. Figure \ref{fig: time_ratio} displays the ratio $T_{f'}/T_f$ between the sampling time $T_{f'}$ (averaged over all realizations of a given depth) for the shallow surrogate network $f'$ and the sampling time $T_f$ (again, averaged) for the deep network $f$. We see that, when $k^*=O(K)$, meaning that the majority of nonlinearity is maintained in $f'$, the sampling times remain roughly the same. On the other hand, when $k^*=O(1)$, meaning the majority of nonlinearity is replaced by affine bounds, the sampling time is reduced by nearly two orders of magnitude, and the reduction follows the expected rate of $k^*/K=O(1/K)$. For in-between surrogate architectures using $k^*=O(\log{K})$ and $k^*=O(\sqrt{K})$, we find respectable time complexity reductions, nearing an order of magnitude decrease in sampling time. The decreases in the lower bound on the probabilistic robustness level are also shown in Figure \ref{fig: time_ratio}. The average lower bound $r_f$ without exploiting structure is $0.1$. Therefore, the degradation of the bound incurred by using the shallow surrogate networks is relatively constant and minimal. The experiment results in the same conclusions when using $\tanh$ activation functions, and when using smaller and larger input set radii $\epsilon_x$.

\begin{figure}[ht]
	\centering
	\includegraphics[width=0.6\linewidth]{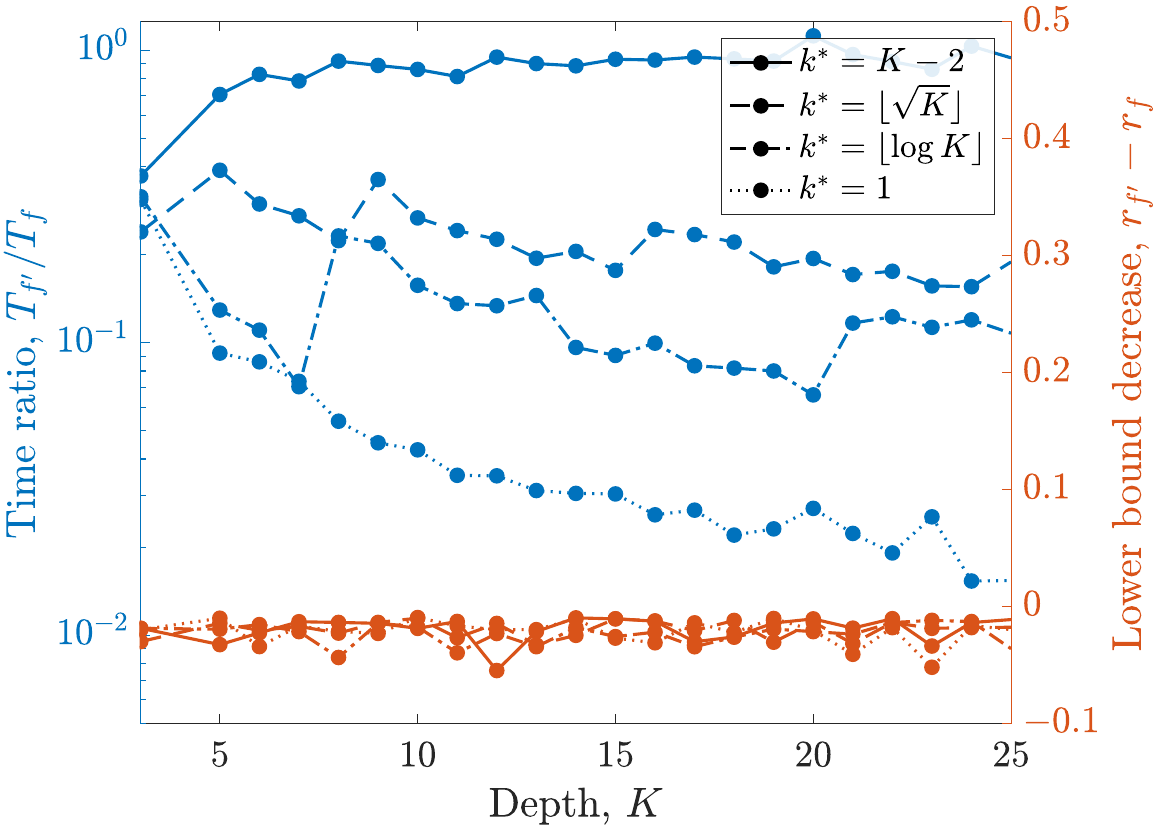}
	\caption{Ratio between the average sampling time of the shallow surrogate network $f'$ and that of the deep original network $f$, and the corresponding decrease in the lower bound on the probabilistic robustness level.}
	\label{fig: time_ratio}
\end{figure}

% Conclusions.

\section{Conclusions}
\label{sec: conclusions}
In this paper, we propose a data-driven method for certifying the robustness of neural networks against random input noise. Sufficient conditions are developed for the convexity of the resulting optimization, as well as on the number of samples to issue a high-probability guarantee for the safety of the output. The method applies to general neural networks and general input noise distributions. In cases where the activation functions can be affinely bounded, we show how to exploit the network structure to reduce sample complexity. The unified framework allows the user to balance the strength of the robustness bound with the tightness of the resulting output set estimate. Our numerical experiments show that the proposed method gives less conservative robustness bounds than the prior state-of-the-art techniques, as it is capable of certifying larger input uncertainty regions on synthetic, MNIST, and CIFAR-10 networks. In situations where neural network failure modes may exist but are unlikely and hence robustness amounts to achieving tolerable risk, these results suggest that re-tooling worst-case analysis techniques from the adversarial robustness literature results in overly conservative bounds. We conclude that taking a data-driven approach to generate probabilistic robustness guarantees, as developed in this paper, is the better option in these contexts.

%%%%% BACK MATTER

% Appendices.
\appendix

\section{Proofs}
\label{app: proofs}

In this appendix, we recall and prove the results stated in the main body of the paper.

\begin{repproposition}{prop: lower_bound_from_epsilon-cover}
	Let $\hat{\mathcal{Y}}$ be an arbitrary subset of $\mathbb{R}^{n_y}$. If $\hat{\mathcal{Y}}$ is an $\epsilon$-cover of $\mathcal{Y}=f(\mathcal{X})$, then
	\begin{equation}
		\hat{r}(\hat{\mathcal{Y}}) \le \bar{r}(\epsilon). \tag{\ref{eq: lower_bound_from_epsilon-cover}}
	\end{equation}
\end{repproposition}
\begin{proof}[Proof of Proposition \ref{prop: lower_bound_from_epsilon-cover}]
	Note that $y\in\hat{\mathcal{Y}}$ implies that $a^\top y + b \ge \hat{r}(\hat{\mathcal{Y}})$ by \eqref{eq: approximate_robustness_level}. Therefore, it holds that $\mathbb{P}_X(f(X)\in\hat{\mathcal{Y}}) \le \mathbb{P}_X(a^\top f(X) + b \ge \hat{r}(\hat{\mathcal{Y}}))$. Since $\hat{\mathcal{Y}}$ is an $\epsilon$-cover of $\mathcal{Y}$, we have that $\mathbb{P}_X(f(X)\in\hat{\mathcal{Y}}) \ge 1-\epsilon$. Hence,
	\begin{equation*}
	1-\epsilon\le\mathbb{P}_X(f(X)\in\hat{\mathcal{Y}})\le\mathbb{P}_X(a^\top f(X) + b \ge \hat{r}(\hat{\mathcal{Y}})).
	\end{equation*}
	This shows that $\hat{r}(\hat{\mathcal{Y}})$ is feasible for the optimization \eqref{eq: probabilistic_robustness_level}. Therefore, $\hat{r}(\hat{\mathcal{Y}})\le\bar{r}(\epsilon)$, as desired.
\end{proof}

\begin{reptheorem}{thm: convex_scenario_optimization}
	Consider the scenario problem \eqref{eq: scenario_problem}. Suppose that $\Theta$ takes the form
	\begin{equation*}
		\Theta = \{\theta\in\mathbb{R}^p : g_i(\theta) \le 0 ~ \text{for all $i\in\{1,2,\dots,m\}$}\},
	\end{equation*}
	where every $g_i\colon\mathbb{R}^p\to\mathbb{R}$ is convex. Furthermore, suppose that $h(\cdot)$ is a concave set-valued function that takes the form
	\begin{equation*}
		h(\theta) = \{y\in\mathbb{R}^{n_y} : h_i(y,\theta) \le 0 ~ \text{for all $i\in\{1,2,\dots,n\}$}\},
	\end{equation*}
	where $h_i\colon\mathbb{R}^{n_y}\times\mathbb{R}^p \to\mathbb{R}$ and $h_i(y,\cdot)$ is convex for all $y\in\mathbb{R}^{n_y}$. Then, \eqref{eq: scenario_problem} is a convex optimization problem.
\end{reptheorem}
\begin{proof}[Proof of Theorem \ref{thm: convex_scenario_optimization}]
	Since \eqref{eq: scenario_problem} is a maximization problem, we must show that under the assumptions on $\Theta$ and $h(\cdot)$, the objective is concave on $\Theta$ and the constraints are convex.

	Let us first consider the objective $\hat{r}(\theta) - \lambda v(\theta)$, where $\hat{r}(\theta) = \inf\{a^\top y + b : y\in h(\theta)\}$. Since
	\begin{enumerate}
		\item $g(y,\theta) \coloneqq a^\top y+b$ is jointly concave on $\mathbb{R}^{n_y}\times\Theta$;
		\item $h(\cdot)$ is a concave set-valued function on $\Theta$;
		\item and $\Theta$ is a convex set;
	\end{enumerate}
	Proposition 3.1 of \cite{fiacco1986convexity} gives that $\hat{r}(\cdot)$ is a concave function on $\Theta$. Since $v(\cdot)$ is assumed to be convex on $\Theta$ and $\lambda\ge 0$, we conclude that the objective is concave.
	
	Now, let us consider the constraints. The constraints $g_i(\theta)\le 0$ are convex, so $\theta\in\Theta$ is a convex constraint. Next, the random constraint $y_j\in h(\theta)$ is equivalent to the constraint on $\theta$ that $h_i(y_j,\theta)\le 0$ for all $i$. Since $h_i(y_j,\cdot)$ is a convex function, the constraint is convex. Since this holds for all $i\in\{1,2,\dots,n\}$ and all $j\in\{1,2,\dots,N\}$, we conclude that all of the constraints in \eqref{eq: scenario_problem} are convex.
\end{proof}

\begin{reptheorem}{thm: high-probability_guarantees}
	Let $\epsilon,\delta\in[0,1]$. Assume that the scenario optimization \eqref{eq: scenario_problem} is convex and is attained by a solution $\theta^*\in\mathbb{R}^p$. If $N \ge \frac{2}{\epsilon}\left( \log\frac{1}{\delta} + p \right)$, then the following inequalities hold:
	\begin{enumerate}
		\item $\mathbb{P}_{\theta^*}(\mathbb{P}_X(f(X)\in h(\theta^*)) \ge 1-\epsilon) \ge 1-\delta$;
		\item $\mathbb{P}_{\theta^*}(\hat{r}(\theta^*) \le \bar{r}(\epsilon)) \ge 1-\delta$.
	\end{enumerate}
\end{reptheorem}
\begin{proof}[Proof of Theorem \ref{thm: high-probability_guarantees}]
	Since the scenario problem is convex and $N\ge \frac{2}{\epsilon}\left( \log\frac{1}{\delta} + p \right)$, Theorem 1 of \cite{campi2009scenario} gives that, with probability at least $1-\delta$, the solution $\theta^*$ is feasible for the chance-constrained problem \eqref{eq: chance-constrained_problem}. Therefore, $\mathbb{P}_{\theta^*}(\mathbb{P}_X(f(X)\in h(\theta^*)) \ge 1-\epsilon) \ge 1-\delta$, which proves the first conclusion.

	To prove the second conclusion, recall the law of total probability: for an arbitrary event $A\in\mathcal{F}_{\theta^*}$ and an arbitrary partition $\{B_1,B_2\}\subseteq\mathcal{F}_{\theta^*}$ of $\Omega_{\theta^*}$ such that $\Omega_{\theta^*} = B_1\cup B_2$, $B_1\cap B_2 = \emptyset$, and $\mathbb{P}_{\theta^*}(B_i)>0$ for $i\in\{1,2\}$, we have that
	\begin{equation*}
	\mathbb{P}_{\theta^*}(A) = \mathbb{P}_{\theta^*}(A|B_1)\mathbb{P}_{\theta^*}(B_1) + \mathbb{P}_{\theta^*}(A|B_2)\mathbb{P}_{\theta^*}(B_2),
	\end{equation*}
	where $\mathbb{P}_{\theta^*}(A|B_1) = \frac{\mathbb{P}_{\theta^*}(A\cap B_1)}{\mathbb{P}_{\theta^*}(B_1)}$ denotes the probability of event $A$ conditioned on event $B_1$, and similarly for $B_2$. Choose the particular events $A = \{\omega\in\Omega_{\theta^*} : \hat{r}(\theta^*(\omega)) \le \bar{r}(\epsilon)\}$, $B_1 = \{\omega\in\Omega_{\theta^*} : \text{$h(\theta^*(\omega))$ is an $\epsilon$-cover of $\mathcal{Y}$}\}$, and $B_2 = \Omega_{\theta^*}\setminus B_1$. Then, Proposition \ref{prop: lower_bound_from_epsilon-cover} shows that  $B_1\subseteq A$, so $\mathbb{P}_{\theta^*}(A|B_1) = \frac{\mathbb{P}_{\theta^*}(B_1)}{\mathbb{P}_{\theta^*}(B_1)} = 1$. Furthermore, by the first conclusion proved above, $\mathbb{P}_{\theta^*}(B_1) = \mathbb{P}_{\theta^*}(\mathbb{P}_X(f(X)\in h(\theta^*))\ge 1-\epsilon) \ge 1-\delta$. Hence, the law of total probability gives that
	\begin{equation*}
	\mathbb{P}_{\theta^*}(A) \ge 1-\delta + \mathbb{P}_{\theta^*}(A|B_2)\mathbb{P}_{\theta^*}(B_2) \ge 1-\delta,
	\end{equation*}
	which proves the second conclusion.
\end{proof}

\begin{replemma}{lem: affine_activation_bounds}
		Suppose that Assumptions \ref{ass: preactivation_bounds} and \ref{ass: affine_activation_bounds} hold. For all $k\in\{1,2,\dots,K\}$, it holds for all $x\in\mathcal{X}$ that
		\begin{equation*}
			\tilde{W}_L^{(k)} \phi^{(k-1)}(x) + \tilde{b}_L^{(k)} \le \phi^{(k)}(x) \le \tilde{W}_U^{(k)} \phi^{(k-1)}(x) + \tilde{b}_U^{(k)},
		\end{equation*}
		where
		\begin{equation}
			\begin{aligned}
				\tilde{W}_L^{(k)} = W_L^{(k)} W^{(k-1)}, &\quad \tilde{b}_L^{(k)} = W_L^{(k)} b^{(k-1)} + b_L^{(k)}, \\
				\tilde{W}_U^{(k)} = W_U^{(k)} W^{(k-1)}, &\quad \tilde{b}_U^{(k)} = W_U^{(k)} b^{(k-1)} + b_U^{(k)}.
			\end{aligned} \tag{\ref{eq: tilde_parameters}}
		\end{equation}
	\end{replemma}
	\begin{proof}[Proof of Lemma \ref{lem: affine_activation_bounds}]
		Let $k\in\{1,2,\dots,K\}$ and let $x\in\mathcal{X}$. Define $z = \mathcal{A}^{(k-1)}\circ \phi^{(k-1)}(x)$. Then, since $z\in[l^{(k)},u^{(k)}]$, it holds that $\mathcal{L}^{(k)}(z) \le \sigma^{(k)}(z) \le \mathcal{U}^{(k)}(z)$. Expanding this inequality using the matrix-vector representation of the affine maps $\mathcal{L}^{(k)},\mathcal{U}^{(k)},\mathcal{A}^{(k-1)}$, we obtain
		\begin{equation*}
			W_L^{(k)}(W^{(k-1)}\phi^{(k-1)}(x) + b^{(k-1)}) + b_L^{(k)} \le \sigma^{(k)}(\mathcal{A}^{(k-1)} \circ \phi^{(k-1)}(x)) \le W_U^{(k)}(W^{(k-1)}\phi^{(k-1)}(x) + b^{(k-1)}) + b_U^{(k)},
		\end{equation*}
		which gives the desired result upon substituting the definitions of $\tilde{W}_L^{(k)},\tilde{W}_U^{(k)},\tilde{b}_L^{(k)},\tilde{b}_U^{(k)}$ and using the fact that $\sigma^{(k)}(\mathcal{A}^{(k-1)}\circ\phi^{(k-1)}(x)) = \phi^{(k)}(x)$.
	\end{proof}

	\begin{lemma}
		\label{lem: bound_affine_function}
		Suppose that $A_1 x_1 + B_1 \le x_2 \le C_1 x_1 + D_1$ and $A_2 x_2 + B_2 \le x_3 \le C_2 x_2 + D_2$ for vectors $x_k,B_k,D_k$ and matrices $A_k,C_k$, all of compatible dimensions. Then
	\begin{equation*}
		E_2x_1 + F_2 \le x_3 \le G_2 x_1 + H_2,
	\end{equation*}
	where
	\begin{align*}
		E_2 &= \min\{0,A_2\}C_1 + \max\{0,A_2\}A_1, \\
		F_2 &= \min\{0,A_2\}D_1 + \max\{0,A_2\}B_1 + B_2, \\
		G_2 &= \max\{0,C_2\}C_1 + \min\{0,C_2\}A_1, \\
		H_2 &= \max\{0,C_2\}D_1 + \min\{0,C_2\}B_1 + D_2.
	\end{align*}
	\end{lemma}

	\begin{proof}[Proof of Lemma \ref{lem: bound_affine_function}]
		Let $(z)_i$ denote the $i^\text{th}$ element of a vector $z$ and $(Z)_{ij}$ denote the $(i,j)^\text{th}$ element of a matrix $Z$. It holds for all indices $i$ that
	\begin{align*}
		(x_3)_i &\le (C_2 x_2 + D_2)_i \\
			&= \sum_{j} (C_2)_{ij}(x_2)_j + (D_2)_i \\
			&= \sum_{j : (C_2)_{ij} \ge 0} (C_2)_{ij}(x_2)_j + \sum_{j : (C_2)_{ij} < 0} (C_2)_{ij} (x_2)_j + (D_2)_i \\
			&\le \sum_{j : (C_2)_{ij} \ge 0} (C_2)_{ij} (C_1 x_1 + D_1)_j + \sum_{j : (C_2)_{ij}<0} (C_2)_{ij} (A_1 x_1 + B_1)_j + (D_2)_i \\
			&= \sum_j (\max\{0,(C_2)_{ij}\}(C_1 x_1 + D_1)_j + \min\{0,(C_2)_{ij}\}(A_1 x_1 + B_1)_j) + (D_2)_i \\
	    &= \big(\max\{0,C_2\}(C_1 x_1 + D_1) + \min\{0,C_2\}(A_1 x_1 + B_1) + D_2 \big)_i,
	\end{align*}
	so
	\begin{equation*}
		x_3\le (\max\{0,C_2\}C_1 + \min\{0,C_2\}A_1)x_1 + \max\{0,C_2\}D_1 + \min\{0,C_2\}B_1 + D_2 = G_2 x_1 + H_2,
	\end{equation*}
	which proves the upper bound on $x_3$.
	
	To prove the lower bound on $x_3$, note that $-x_3 \le (-A_2)x_2 + (-B_2)$, so the above analysis yields that
	\begin{align*}
		-x_3 &\le (\max\{0,-A_2\}C_1 + \min\{0,-A_2\}A_1)x_1 + \max\{0,-A_2\}D_1 + \min\{0,-A_2\}B_1 -B_2 \\
		     &= -(\min\{0,A_2\}C_1 + \max\{0,A_2\}A_1)x_1 - (\min\{0,A_2\}D_1 + \max\{0,A_2\}B_1 + B_2) \\
		     &= -E_2 x_1 - F_2,
	\end{align*}
	which concludes the proof.
	\end{proof}

	\begin{lemma}
		\label{lem: bound_affine_functions}
		Let $M\in\mathbb{N}$, $M>1$. Suppose that
		\begin{equation*}
			A_n x_n + B_n \le x_{n+1} \le C_n x_n + D_n
		\end{equation*}
		for all $n\in\{1,2,\dots,M\}$, where the vectors $x_n,B_n,D_n$ and the matrices $A_n,C_n$ are all of compatible dimensions. For all $n\in\{2,3,\dots,M\}$, define
		\begin{align*}
			E_n &= \min\{0,A_n\}G_{n-1} + \max\{0,A_n\}E_{n-1}, \\
			F_n &= \min\{0,A_n\}H_{n-1} + \max\{0,A_n\}F_{n-1} + B_n, \\
			G_n &= \max\{0,C_n\}G_{n-1} + \min\{0,C_n\}E_{n-1}, \\
			H_n &= \max\{0,C_n\}H_{n-1} + \min\{0,C_n\}F_{n-1} + D_n,
		\end{align*}
		where $E_1 = A_1$, $F_1 = B_1$, $G_1 = C_1$, and $H_1 = D_1$. Then
		\begin{equation}
			E_n x_1 + F_n \le x_{n+1} \le G_n x_1 + H_n \label{eq: bound_affine_functions}
		\end{equation}
		holds for all $n\in\{1,2,\dots,M\}$.
	\end{lemma}

	\begin{proof}[Proof of Lemma \ref{lem: bound_affine_functions}]
		The result holds for $n=1$ by assumption. We prove the result for $n\in\{2,3,\dots,M\}$ by induction on $n$. Lemma \ref{lem: bound_affine_function} shows that the result holds for the base case $n=2$. Now, suppose that the result holds for some arbitrary $n\in\{2,3,\dots,M-1\}$, so that
		\begin{equation}
			E_n x_1 + F_n \le x_{n+1} \le G_n x_1 + H_n. \label{eq: bound_affine_functions_1}
		\end{equation}
		Make the following definitions:
		\begin{gather*}
		x_1' = x_1, \quad x_2' = x_{n+1}, \quad x_3' = x_{n+2}, \\
		A_1' = E_n, \quad B_1' = F_n, \quad C_1' = G_n, \quad D_1' = H_n, \\
		A_2' = A_{n+1}, \quad B_2' = B_{n+1}, \quad C_2' = C_{n+1}, \quad D_2' = D_{n+1}.
		\end{gather*}
		Then by the assumption that $A_{n+1}x_{n+1} + B_{n+1} \le x_{n+2} \le C_{n+1} x_{n+1} + D_{n+1}$ it holds that
		\begin{equation}
			A_2'x_2' + B_2' \le x_3' \le C_2'x_2' + D_2'. \label{eq: bound_affine_functions_2}
		\end{equation}
		Also, by the induction hypothesis \eqref{eq: bound_affine_functions_1}, it holds that
		\begin{equation}
			A_1'x_1' + B_1' \le x_2' \le C_1'x_1'+D_1'. \label{eq: bound_affine_functions_3}
		\end{equation}
		Therefore, \eqref{eq: bound_affine_functions_2} and \eqref{eq: bound_affine_functions_3} together with Lemma \ref{lem: bound_affine_function} give that
		\begin{equation}
			E_2'x_1' + F_2' \le x_3' \le G_2' x_1' + H_2', \label{eq: bound_affine_functions_4}
		\end{equation}
		where
		\begin{align*}
		E_2' &= \min\{0,A_2'\}C_1' + \max\{0,A_2'\}A_1', \\
		F_2' &= \min\{0,A_2'\}D_1' + \max\{0,A_2'\}B_1' + B_2', \\
		G_2' &= \max\{0,C_2'\}C_1' + \min\{0,C_2'\}A_1', \\
		H_2' &= \max\{0,C_2'\}D_1' + \min\{0,C_2'\}B_1' + D_2'.
	\end{align*}
	Substituting our earlier definitions for these values gives that $E_2' = E_{n+1}$, $F_2' = F_{n+1}$, $G_2' = G_{n+1}$, and $H_2' = H_{n+1}$, and therefore in light of the fact that $x_1'=x_1$ and $x_3'=x_{n+2}$, \eqref{eq: bound_affine_functions_4} becomes
	\begin{equation*}
		E_{n+1}x_1 + F_{n+1} \le x_{n+2} \le G_{n+1}x_1 + H_{n+1},
	\end{equation*}
	so the induction step has been proven. Thus, the result \eqref{eq: bound_affine_functions} holds for all $n\in\{1,2,\dots,M\}$.
	\end{proof}

	\begin{repproposition}{prop: replacement}
	Suppose that Assumptions \ref{ass: preactivation_bounds} and \ref{ass: affine_activation_bounds} hold, and assume that $K\ge 3$. Let $k^*\in\{1,2,\dots,K-2\}$ and define $M=K-k^*$. Consider the matrices $\tilde{W}_L^{(k)},\tilde{W}_U^{(k)}$ and vectors $\tilde{b}_L^{(k)},\tilde{b}_U^{(k)}$ defined in \eqref{eq: tilde_parameters}. Define $E_1 = \tilde{W}_L^{(k^*+1)}$, $F_1 = \tilde{b}_L^{(k^*+1)}$, $G_1 = \tilde{W}_U^{(k^*+1)}$, and $H_1 = \tilde{b}_U^{(k^*+1)}$. Also, for $n\in\{2,3,\dots,M\}$, define
		\begin{align*}
			E_n &= \min\{0,\tilde{W}_L^{(k^*+n)}\} G_{n-1} + \max\{0,\tilde{W}_L^{(k^*+n)}\} E_{n-1}, \\
			F_n &= \min\{0,\tilde{W}_L^{(k^*+n)}\} H_{n-1} + \max\{0,\tilde{W}_L^{(k^*+n)}\} F_{n-1} + \tilde{b}_L^{(k^*+n)}, \\
			G_n &= \max\{0,\tilde{W}_U^{(k^*+n)}\} G_{n-1} + \min\{0,\tilde{W}_U^{(k^*+n)}\} E_{n-1}, \\
			H_n &= \max\{0,\tilde{W}_U^{(k^*+n)}\} H_{n-1} + \min\{0,\tilde{W}_U^{(k^*+n)}\} F_{n-1} + \tilde{b}_U^{(k^*+n)}.
		\end{align*}
	Then, for all $x\in\mathcal{X}$, it holds that
	\begin{equation*}
		E_M \phi^{(k^*)}(x) + F_M \le \phi^{(K)}(x) \le G_M \phi^{(k^*)}(x) + H_M.
	\end{equation*}
	\end{repproposition}

	\begin{proof}[Proof of Proposition \ref{prop: replacement}]
		Let $x\in\mathcal{X}$. Then, by Lemma \ref{lem: affine_activation_bounds}, it holds that
		\begin{equation}
			\tilde{W}_L^{(k)}\phi^{(k-1)}(x) + \tilde{b}_L^{(k)} \le \phi^{(k)}(x) \le \tilde{W}_U^{(k-1)}\phi^{(k-1)}(x) + \tilde{b}_U^{(k)} \label{eq: replacement_1}
		\end{equation}
		for all $k\in\{1,2,\dots,K\}$. For all $n\in\{1,2,\dots,M+1\}$, define
		\begin{equation*}
		x_n = \phi^{(k^*+n-1)}(x).
		\end{equation*}
		Also, for all $n\in\{1,2,\dots,M\}$, define
		\begin{equation*}
			A_n = \tilde{W}_L^{(k^*+n)}, \quad B_n = \tilde{b}_L^{(k^*+n)}, \quad C_n = \tilde{W}_U^{(k^*+n)}, \quad D_n = \tilde{b}_U^{(k^*+n)}.
		\end{equation*}
		Then it holds that $E_1 = A_1$, $F_1=B_1$, $G_1=C_1$, $H_1=D_1$, and
		\begin{align*}
			E_n &= \min\{0,A_n\}G_{n-1} + \max\{0,A_n\}E_{n-1}, \\
			F_n &= \min\{0,A_n\}H_{n-1} + \max\{0,A_n\}F_{n-1} + B_n, \\
			G_n &= \max\{0,C_n\}G_{n-1} + \min\{0,C_n\}E_{n-1}, \\
			H_n &= \max\{0,C_n\}H_{n-1} + \min\{0,C_n\}F_{n-1} + D_n,
		\end{align*}
		for $n\in\{2,3,\dots,M\}$. Also, \eqref{eq: replacement_1} gives that
		\begin{equation*}
			A_n x_n + B_n \le x_{n+1} \le C_n x_n + D_n
		\end{equation*}
		for all $n\in\{1,2,\dots,M\}$, so by Lemma \ref{lem: bound_affine_functions}, we conclude that
		\begin{equation*}
		E_n x_1 + F_n \le x_{n+1} \le G_n x_1 + H_n
		\end{equation*}
		for all $n\in\{1,2,\dots,M\}$. In particular, for $n=M$, this yields the following bound on $x_{M+1}=\phi^{(K)}(x)$ in terms of $x_1 = \phi^{(k^*)}(x)$:
		\begin{equation*}
			E_M \phi^{(k^*)}(x) + F_M \le \phi^{(K)}(x) \le G_M \phi^{(k^*)}(x) + H_M,
		\end{equation*}
		which is the desired result.
	\end{proof}

\section{Extension to General Polyhedral Safe Sets}
\label{sec: extension_to_general_polyhedral_safe_sets}
In this section, we explicitly walk through the steps of generalizing our proposed assessment method to the case where the safe set is a general polyhedral set defined by the intersection of finitely many half-spaces.

Consider the polyhedral safe set $\mathcal{S} = \{y\in\mathbb{R}^{n_y} : Ay + b\ge 0\}$, where $A\in\mathbb{R}^{n_s\times n_y}$ and $b\in\mathbb{R}^{n_s}$. Denote the $i^\text{th}$ row of $A$ by $a_i^\top$ and the $i^\text{th}$ element of $b$ by $b_i$. In this setting, the condition $y\in\mathcal{S}$ is equivalent to $\min_{i\in\{1,2,\dots,n_s\}} a_i^\top y + b_i \ge 0$. Therefore, the deterministic robustness level is naturally formulated as
\begin{equation*}
	r^* = \inf_{y\in\mathcal{Y}} \min_{i\in\{1,2,\dots,n_s\}} a_i^\top y + b_i,
\end{equation*}
so that $r^* \ge 0$ certifies that $Y=f(X)$ is safe with probability one. Then the approximate robustness level using a surrogate output set $\hat{\mathcal{Y}}$ becomes
\begin{equation*}
	\hat{r}(\hat{\mathcal{Y}}) = \inf_{y\in\hat{\mathcal{Y}}} \min_{i\in\{1,2,\dots,n_s\}} a_i^\top y + b_i.
\end{equation*}
Moreover, the condition that $f(X)$ has safety level at least $r$ with high probability is naturally encoded in the following probabilistic robustness level:
\begin{equation*}
	\bar{r}(\epsilon) = \sup\left\{r\in\mathbb{R} : \mathbb{P}_X\left(\min_{i\in\{1,2,\dots,n_s\}} a_i^\top f(X) + b_i \ge r\right) \ge 1-\epsilon\right\}.
\end{equation*}

With our robustness levels defined for the general polyhedral safe set, we now outline the procedure to generalize our main results from the single half-space case presented in the paper. To this end, start by prescribing probability levels $\epsilon,\delta\in[0,1]$ close to zero, and define $\epsilon'=\epsilon/n_s$ and $\delta' = \delta / n_s$. Then, for all $i\in\{1,2,\dots,n_s\}$, perform the proposed assessment method for the single half-space setting using the parameters $\epsilon'$, $\delta'$, $a_i$, and $b_i$ in place of $\epsilon$, $\delta$, $a$, and $b$, respectively. In particular, for every $i$, use $N' \ge \frac{2}{\epsilon'}\left( \log\frac{1}{\delta'} + p\right)$ independent and identically distributed samples in the scenario problem \eqref{eq: scenario_problem}, ensuring that the samples across different values of $i$ are also independent. Notice that the sample size $N'$ grows with $n_s$ like $n_s\log(n_s)$, so the increase in the number of samples used for each scenario problem is modest, as it is nearly linear. However, the total number of samples needed across all $n_s$ scenario problems grows with $n_s$ like $n_s^2\log(n_s)$, and therefore the computational cost may become prohibitive in the case the safe set is defined by a large number of half-spaces. To remedy this, one may first compute a polyhedral inner-approximation of $\mathcal{S}$ with a much smaller number of half-spaces, and then apply the methods outlined in this section.

Now, let $\theta^*_i$ denote the solution to the scenario problem \eqref{eq: scenario_problem} corresponding to row $i$ of the safe set. Remark that the solutions $\theta_i^*$ are all random, although they are not necessarily defined on the same probability space, as their distributions depend on the particular values for $a_i$ and $b_i$ used to compute them. For notational convenience, denote the probability distribution of $\theta_i^*$ by $\mathbb{P}_i$, and denote by $\mathbb{P}$ the product probability measure associated with $(\theta_1^*,\theta_2^*,\dots,\theta_{n_s}^*)$. Then, Theorem \ref{thm: high-probability_guarantees} gives for all $i$ that, with probability at least $1-\delta'$, the set $h(\theta^*_i)$ is an $\epsilon'$-cover of $\mathcal{Y}=f(\mathcal{X})$.
That is,
\begin{equation*}
	\mathbb{P}_i(\mathbb{P}_X(f(X)\in h(\theta_i^*))\ge 1-\epsilon') \ge 1-\delta'.
\end{equation*}
Note that if $\hat{\mathcal{Y}}_i$ are $\frac{\epsilon}{n_c}$-covers of $\mathcal{Y}$ for all $i\in\{1,2,\dots,n_c\}$, then $\bigcap_{i=1}^{n_c} \hat{\mathcal{Y}}_i$ is an $\epsilon$-cover, since
\begin{align*}
\mathbb{P}_X\left(f(X)\in \bigcap_{i=1}^{n_c}\hat{\mathcal{Y}}_i\right) &= 1 - \mathbb{P}_X\left( f(X) \in \bigcup_{i=1}^{n_c} \hat{\mathcal{Y}}_i^c \right) \\
									&\ge 1 - \sum_{i=1}^{n_c} \mathbb{P}_X(f(X)\in \hat{\mathcal{Y}}_i^c) \\
									&= 1 - \sum_{i=1}^{n_c} (1-\mathbb{P}_X(f(X)\in\hat{\mathcal{Y}}_i)) \\
									&\ge 1-\sum_{i=1}^{n_c}\frac{\epsilon}{n_c} \\
									&=1-\epsilon,
\end{align*}
where $\hat{\mathcal{Y}}_i^c$ denotes the complement $\mathbb{R}^{n_y}\setminus \hat{\mathcal{Y}}_i$. Using the monotonicity and subadditivity of the product measure $\mathbb{P}$, we can apply this result to the sets $h(\theta_i^*)$ to find that
\begin{align*}
	\mathbb{P}\left( \mathbb{P}_X\left( f(X)\in \bigcap_{i=1}^{n_s} h(\theta_i^*) \right) \ge 1-\epsilon \right) &\ge \mathbb{P}( \mathbb{P}_X\left( f(X)\in h(\theta_i^*) \right) \ge 1-\epsilon' ~ \text{for all $i$} ) \\
									&= 1-\mathbb{P}(\mathbb{P}_X(f(X)\in h(\theta_i^*))<1-\epsilon' ~ \text{for some $i$}) \\
									&\ge 1 - \sum_{i=1}^{n_s}\mathbb{P}_i(\mathbb{P}_X(f(X)\in h(\theta_i^*))<1-\epsilon')\\
									&= 1 - \sum_{i=1}^{n_s}(1-\mathbb{P}_i(\mathbb{P}_X(f(X)\in h(\theta_i^*))\ge 1-\epsilon')) \\
									&\ge 1-\sum_{i=1}^{n_s}\delta' \\
									&= 1-\delta.
\end{align*}
Therefore, we conclude that $\bigcap_{i=1}^{n_s} h(\theta_i^*)$ is an $\epsilon$-cover of $\mathcal{Y}$ with probability at least $1-\delta$.

Note that every individual $h(\theta_i^*)$ is an $\epsilon$-cover of $\mathcal{Y}$ with probability at least $1-\delta$ as well, since it is an $\epsilon'$-cover with probability at least $1-\delta'\ge 1-\delta$ by construction, and an $\epsilon'$-cover is certainly an $\epsilon$-cover since $\epsilon'\le \epsilon$. However, it is important to remark that the intersection $\bigcap_{i=1}^{n_s} h(\theta_i^*)$ is clearly a tighter $\epsilon$-cover than any of the individual covers $h(\theta_i^*)$, and therefore it gives better probabilistic localization of the output. Furthermore, recall that the $\epsilon$-cover is used as a surrogate output set to compute the approximate robustness level in order to lower-bound the probabilistic robustness level. Therefore, we'd like to choose the $\epsilon$-cover so that the approximate robustness level is maximal. Since
\begin{align*}
	\hat{r}\left( \bigcap_{i=1}^{n_s} h(\theta_i^*) \right) &= \inf\left\{ s(y) : y\in \bigcap_{i=1}^{n_s} h(\theta_i^*) \right\} \\
								&\ge \inf \bigcap_{i=1}^{n_s} \{s(y) : y\in h(\theta_i^*)\} \\
								&\ge \max_{i\in\{1,2,\dots,n_s\}} \inf\{s(y) : y\in h(\theta_i^*)\} \\
								&= \max_{i\in\{1,2,\dots,n_s\}} \hat{r}(h(\theta_i^*)),
\end{align*}
where $s(y) = \min_{i\in\{1,2,\dots,n_s\}} a_i^\top y + b_i$ is the safety level of $y$ with respect to the general polyhedral safe set $\mathcal{S}$, it is clear that using the intersection $\bigcap_{i=1}^{n_s} h(\theta_i^*)$ will give a tighter bound on the probabilistic robustness level than any of the individual covers $h(\theta_i^*)$.

Since we know that $\bigcap_{i=1}^{n_s} h(\theta_i^*)$ is an $\epsilon$-cover with probability $1-\delta$, the only result that remains to be generalized is the second conclusion from Theorem \ref{thm: high-probability_guarantees}, i.e., we want to formally guarantee that $\hat{r}\left( \bigcap_{i=1}^{n_s} h(\theta_j^*) \right) \le \bar{r}(\epsilon)$ with probability $1-\delta$. This follows readily from the fact that $\mathbb{P}\left( \mathbb{P}_X\left( f(X)\in \bigcap_{i=1}^{n_s} h(\theta_i^*) \right) \ge 1-\epsilon \right) \ge 1-\delta$ together with Proposition \ref{prop: lower_bound_from_epsilon-cover} and the law of total probability, just as in the proof of Theorem \ref{thm: high-probability_guarantees} in Appendix \ref{app: proofs}.

Finally, note that when the class $\mathcal{H}$ of surrogate output sets is designed so that $h(\theta)$ is convex for all $\theta\in \Theta$, the value $\hat{r}\left( \bigcap_{i=1}^{n_s}h(\theta_i^*) \right) = \inf_{y\in\bigcap_{j=1}^{n_s} h(\theta_j^*)} \min_{i\in\{1,2,\dots,n_s\}} a_i^\top y + b_i = \min_{i\in\{1,2,\dots,n_s\}} \inf_{y\in\bigcap_{j=1}^{n_s} h(\theta_j^*)} a_i^\top y + b_i$ is easily computable, as it involves $n_s$ minimizations of affine functions over the convex set $\bigcap_{i=1}^{n_s} h(\theta_i^*)$. This completes the generalization of our assessment method to the case with a general polyhedral safe set.

\section{Distributionally Robust Extension}
\label{sec: distributionally_robust_extension}

In this section, we formulate a distributionally robust variant of the proposed assessment procedure. Consider the case where the neural network input $X$ has a finite number of known possible probability distributions $\mathbb{P}_1,\mathbb{P}_2,\dots,\mathbb{P}_q$, but that at any point in time, the true input distribution $\mathbb{P}_X$ is unknown. For simplicity, we use the notation $\mathbb{P}_k(P(X))$, where $P$ is a mathematical predicate, to mean the probability of the event $P(X)$ when $X$ is distributed according to $\mathbb{P}_k$. Naturally, we formulate the following distributionally robust variant to the chance-constrained problem \eqref{eq: chance-constrained_problem}:
\begin{equation*}
	\begin{aligned}
	& \underset{\theta\in\Theta}{\text{maximize}} && \hat{r}(\theta) - \lambda v(\theta) \\
	& \text{subject to} && \min_{k\in\{1,2,\dots,q\}} \mathbb{P}_k(f(X)\in h(\theta)) \ge 1-\epsilon.
	\end{aligned}
\end{equation*}
As we did before, we consider a scenario-based approximation to the above chance-constrained problem. However, instead of directly analyzing the above problem, consider treating each distribution separately. That is, let $\delta' = \delta/q$, and for all $k\in\{1,2,\dots,q\}$, formulate the scenario problem
\begin{equation*}
	\begin{aligned}
	& \underset{\theta\in\Theta}{\text{maximize}} && \hat{r}(\theta) - \lambda v(\theta) \\
	& \text{subject to} && y_{j,k}\in h(\theta) ~ \text{for all $j\in\{1,2,\dots,N'\}$},
	\end{aligned}
\end{equation*}
where we take the sample size to be $N' = \lceil \frac{2}{\epsilon}\left( \log\frac{1}{\delta'} + p \right)\rceil$, the samples $x_{1,k},x_{2,k},\dots,x_{N',k}$ are drawn independently and identically from $\mathbb{P}_k$, and $y_{j,k} = f(x_{j,k})$. Denote the solution to the $k^\text{th}$ such scenario problem as $\theta^*_k$ and its associated probability distribution as $\mathbb{P}_{\theta^*_k}$. Also, denote by $\mathbb{P}_{\theta^*}$ the product probability measure associated with $(\theta_1^*,\theta_2^*,\dots,\theta_q^*)$. Then by Theorem \ref{thm: high-probability_guarantees}, we have for all $k$ that $\mathbb{P}_{\theta^*_k}(\mathbb{P}_k(f(X)\in h(\theta^*_k))\ge 1-\epsilon) \ge 1-\delta'$. Therefore, performing a similar line of analysis as in Appendix \ref{sec: extension_to_general_polyhedral_safe_sets}, we find that
\begin{align*}
	\mathbb{P}_{\theta^*}\left( \mathbb{P}_X\left( f(X) \in \bigcup_{k=1}^{q} h(\theta_k^*) \right) \ge 1-\epsilon \right) & \ge \mathbb{P}_{\theta^*}\left( \mathbb{P}_j\left( f(X)\in \bigcup_{k=1}^q h(\theta_k^*) \right) \ge 1-\epsilon ~ \text{for all $j$} \right) \\
	&\ge \mathbb{P}_{\theta^*}\left( \mathbb{P}_j(f(X)\in h(\theta_j^*)) \ge 1-\epsilon ~ \text{for all $j$} \right) \\
	&= 1-\mathbb{P}_{\theta^*}\left( \mathbb{P}_j(f(X)\in h(\theta_j^*)) <1-\epsilon ~ \text{for some $j$} \right) \\
	&\ge 1-\sum_{j=1}^q \mathbb{P}_{\theta_j^*}\left(\mathbb{P}_j(f(X)\in h(\theta_j^*)) < 1-\epsilon\right) \\
	&\ge 1-\sum_{j=1}^q \delta' \\
	&= 1-\delta.
\end{align*}
Therefore, the set $\bigcup_{k=1}^q h(\theta_k^*)$ is an $\epsilon$-cover of $\mathcal{Y}$ with probability at least $1-\delta$. Note that this distributionally robust approach naturally leads to the union of precomputed covers, in contrast to the intersection found in Appendix \ref{sec: extension_to_general_polyhedral_safe_sets}. This is to be expected, since each set $h(\theta_k^*)$ in the current discussion brings new information about where outputs could be located when the input is distributed according to $\mathbb{P}_k$, and this new information should be included in the final $\epsilon$-cover so as to ensure good localization of the output in a distributionally robust sense. In contrast, since all of the samples in Appendix \ref{sec: extension_to_general_polyhedral_safe_sets} come from the same distribution, it is reasonable to intersect the resulting output set estimates and still obtain a good estimate of the true output set $\mathcal{Y}$ with high-probability localization guarantees.

Now, following the same analysis as in Appendix \ref{sec: extension_to_general_polyhedral_safe_sets}, it is easy to see that, with probability $1-\delta$, the probabilistic robustness level is lower-bounded as $\hat{r}\left( \bigcup_{k=1}^q h(\theta_k^*) \right) \le \bar{r}(\epsilon)$, where $\hat{r}(\cdot)$ and $\bar{r}(\cdot)$ are as defined in \eqref{eq: approximate_robustness_level} and \eqref{eq: probabilistic_robustness_level}, respectively. Finally, note that computing $\hat{r}\left( \bigcup_{k=1}^q h(\theta_k^*) \right)$ is simple, since
\begin{align*}
	\hat{r}\left(\bigcup_{k=1}^q h(\theta_k^*) \right) &= \inf \left\{a^\top y + b : y\in\bigcup_{k=1}^q h(\theta_k^*) \right\} \\
							   &= \inf \bigcup_{k=1}^q \{a^\top y + b : y\in h(\theta^*_k)\} \\
							   &= \min_{k\in\{1,2,\dots,q\}} \inf\{a^\top y + b : y\in h(\theta_k^*)\} \\
							   &= \min_{k\in\{1,2,\dots,q\}} \hat{r}(\theta_k^*),
\end{align*}
and the values $\hat{r}(\theta_k^*)$ have already been computed using convex optimization.

\section{Special Case: Class of Half-Spaces}
\label{sec: special_case-class_of_half-spaces}

In this section, we consider the special case of the scenario problem \eqref{eq: scenario_problem} where $\lambda=0$, $\Theta = \mathbb{R}^{n_y}\times\mathbb{R}$, and $h\colon\Theta\to\mathcal{P}(\mathbb{R}^{n_y})$ is given by $h(c,d) = \{y\in\mathbb{R}^{n_y} : c^\top y + d \ge 0\}$. Then $\mathcal{H}$ is the class of all half-spaces within the output space $\mathbb{R}^{n_y}$. We will show that, 1) the scenario problem has a closed-form solution, and 2) the scenario problem coincides with the optimization obtained by applying the scenario approach directly to the definition of $\bar{r}(\epsilon)$.

Under the given conditions, the approximate robustness level becomes $\hat{r}(c,d) = \inf\{a^\top y + b : c^\top y + d \ge 0\}$. The Lagrangian for this minimization problem is
\begin{align*}
	L(y,\mu) = a^\top y + b - \mu(c^\top y + d) = (a-\mu c)^\top y + b-\mu d,
\end{align*}
where $\mu\ge 0$ denotes the Lagrange multiplier. Since the Lagrangian is affine in $y$, the dual function is
\begin{equation*}
	g(\mu) = \inf_{y\in\mathbb{R}^{n_y}} L(y,\mu) = \begin{aligned}
	\begin{cases}
	b-\mu d & \text{if $a=\mu c$}, \\
	-\infty & \text{otherwise}.
	\end{cases}
	\end{aligned}
\end{equation*}
Therefore, the dual problem corresponding to the primal minimization over $y$ becomes
\begin{equation*}
	\begin{aligned}
	& \underset{\mu\in\mathbb{R}}{\text{maximize}} && b-\mu d \\
	& \text{subject to} && a=\mu c, ~ \mu\ge 0.
	\end{aligned}
\end{equation*}
Now, since the primal problem over $y$ is a feasible linear program, we have that strong duality holds \cite{boyd2004convex}. Hence, $\hat{r}(c,d) = \sup\{b-\mu d : a=\mu c, ~ \mu\ge 0\}$. Therefore, the scenario optimization problem \eqref{eq: scenario_problem} reduces to
\begin{equation}
	\begin{aligned}
	& \underset{c\in\mathbb{R}^{n_y}, ~ d,\mu\in\mathbb{R}}{\text{maximize}} && b-\mu d \\
	& \text{subject to} && a=\mu c, ~ \mu \ge 0, \\
	&&& c^\top y_j + d \ge 0 ~ \text{for all $j\in\{1,2,\dots,N\}$}.
	\end{aligned} \label{eq: half-space_scenario_problem}
\end{equation}

We now solve the scenario problem \eqref{eq: half-space_scenario_problem} in closed form. First, under the assumption that the safe set is nontrivial, i.e., $a\ne 0$, we remark that the constraint $a=\mu c$ implies that $\mu\ne 0$, and therefore can be rewritten as $c = \frac{1}{\mu} a$. Eliminating $c$, the optimization becomes
\begin{equation*}
	\begin{aligned}
	& \underset{d,\mu\in\mathbb{R}}{\text{maximize}} && b-\mu d \\
	& \text{subject to} && \mu > 0, ~ \frac{1}{\mu} a^\top y_j + d \ge 0, ~j\in\{1,2,\dots,N\}.
	\end{aligned}
\end{equation*}
Defining $\tilde{r} = b-\mu d$, the problem further reduces to
\begin{equation}
	\begin{aligned}
	& \underset{\tilde{r}\in\mathbb{R}}{\text{maximize}} && \tilde{r} \\
	& \text{subject to} && a^\top y_j + b \ge \tilde{r} ~ \text{for all $j\in\{1,2,\dots,N\}$}.
	\end{aligned} \label{eq: half-space_scenario_problem_reduced}
\end{equation}
It is clear that the reduced scenario problem \eqref{eq: half-space_scenario_problem_reduced} matches the formulation obtained by directly applying the scenario approach to estimate $\bar{r}(\epsilon) = \sup\{r\in\mathbb{R} : \mathbb{P}_X(a^\top f(X) + b \ge r) \ge 1-\epsilon\}$. In fact, since the optimization defining $\bar{r}(\epsilon)$ is univariate, whereas the scenario problem \eqref{eq: half-space_scenario_problem} over $\Theta$ is $(n_y+1)$-dimensional, the number of samples indicated by Theorem \ref{thm: high-probability_guarantees} is conservative for this problem. Instead of $\frac{2}{\epsilon}(\log\frac{1}{\delta}+n_y+1)$ samples, only $N\ge \frac{2}{\epsilon}(\log\frac{1}{\delta}+1)$ samples are needed to obtain the high-probability guarantees provided by Theorem \ref{thm: high-probability_guarantees}. We also note that the optimization \eqref{eq: half-space_scenario_problem_reduced} is a univariate linear program, and is clearly solved in closed-form by $\tilde{r}^* = \min_{j\in\{1,2,\dots,N\}} a^\top y_j + b$. This derivation results in the following proposition:
\begin{proposition}
	Let $\epsilon,\delta\in[0,1]$, $N\ge \frac{2}{\epsilon}(\log\frac{1}{\delta}+1)$, and $\{x_j : j\in\{1,2,\dots,N\}\}$ be a set of $N$ independently and identically distributed samples drawn from $\mathbb{P}_X$. Let $y_j = f(x_j)$ for all $j\in\{1,2,\dots,N\}$. Then with probability $1-\delta$, the probabilistic robustness level $\bar{r}(\epsilon)$ is lower-bounded by $\tilde{r}^* = \min_{j\in\{1,2,\dots,N\}} a^\top y_j + b$.
\end{proposition}

Despite being derived from our general framework, this special case reduces to a solution that is remarkably simple and coincides with a heuristic one may first try using in practice. That is, upon choosing $\mathcal{H}$ to be the class of half-spaces, the optimal sample-based method for lower-bounding the probabilistic robustness level via the approximate robustness level is to compute the minimum safety level amongst the collection of sampled outputs. If sufficiently many samples are used and the minimum safety level is nonnegative, then we certify with high probability that the unknown random output $Y=f(X)$ is safe in practice. Although this special case does not yield meaningful localization of the outputs and uses a crude class of surrogate output sets, it certainly provides a fast analytical method for certifying the network's robustness against random input noise. Furthermore, our derivation mathematically justifies the use of this otherwise heuristic method, and, contrarily, the natural intuition behind this statistical estimator validates our framework that generalizes it.

\section{Alternate Illustrative Example}
\label{sec: alternate_illustrative_example}

In this section, we showcase an illustrative example similar to that in Section \ref{sec: illustrative_example}, albeit now we consider the multilayer perceptron proposed in Example 2 in \cite{xiang2018output}. The network uses a $\tanh(\cdot)$ activation function and has two inputs and two outputs, making the visualization of the input and output sets possible. See \cite{xiang2018output} for more details of the network. The noisy input is distributed uniformly on $\mathcal{X} = \{x\in\mathbb{R}^2 : |x_1 - 0.5|\le 1.5, ~ |x_2-0.5|\le 0.1\}$, where $\bar{x}=(0.5,0.5)$ is the nominal input. The safe set is $\mathcal{S}_1 = \{y\in\mathbb{R}^2 : a^\top y + b \ge 0\}$, where $a=(1,0)$ and $b = 3.7$.

\begin{comment}
% Commenting sentence below for space reasons in TCNS paper.
We now illustrate our proposed robustness certification procedure to localize the random output $Y=f(X)$ and assess its safety.
\end{comment}
The norm ball class $\mathcal{H}$ of Examples \ref{ex: norm_ball_class}, \ref{ex: norm_ball_functions_are_affine}, and \ref{ex: scenario_optimization_with_norm_ball_class} is employed with $\|\cdot\|$ being the $\ell_2$-norm.
\begin{comment}
% Commenting sentence below for space in TCNS paper.
As shown in Example \ref{ex: norm_ball_functions_are_affine}, $h(\cdot)$ is an affine set-valued function, and therefore $\Theta$ and $h(\cdot)$ satisfy the conditions of Theorem \ref{thm: convex_scenario_optimization}.
\end{comment}
The probability levels are chosen as $\epsilon = 0.1$ and $\delta = 10^{-5}$. We set $N=\left\lceil{\frac{2}{\epsilon}(\log\frac{1}{\delta}+p)}\right\rceil=291$, then uniformly sample $N$ inputs $x_j$ from $\mathcal{X}$ and compute their corresponding outputs $y_j$. As shown in Example \ref{ex: scenario_optimization_with_norm_ball_class}, the scenario problem takes the form given in \eqref{eq: scenario_optimization_with_norm_ball_class}. We choose the regularizer to be the square of the norm ball radius, i.e., $v(\bar{y},r)=r^2$. The optimization problem is convex as guaranteed by Theorem \ref{thm: convex_scenario_optimization}.

We solve the scenario problem first without regularization, and then with two different levels of regularization: $\lambda_1=0.001$ and $\lambda_2=1$. The respective solutions are denoted by $\theta^*$, $\theta^*_{\lambda_1}$, and $\theta^*_{\lambda_2}$. Each instance takes approximately $5$ seconds to solve using CVX in \textsc{Matlab} on a standard laptop with a $\SI{2.6}{\giga\hertz}$ dual-core i5 processor. The resulting approximate robustness levels are $\hat{r}(\theta^*) = 0.423$, $\hat{r}(\theta^*_{\lambda_1}) = 0.419$, and $\hat{r}(\theta^*_{\lambda_2}) = -1.107$. In the instances without regularization and with regularization level $\lambda_1$, Theorem \ref{thm: high-probability_guarantees} guarantees that the probabilistic robustness level $\bar{r}(0.1)$ is at least $0.4$ with probability at least $0.99999$. In other words, the random output $Y=f(X)$ has a safety level of $0.4$ with high probability, granting the probablistic robustness certificate we seek. On the other hand, since $\hat{r}(\theta_{\lambda_2}^*)<0$, the scenario problem using regularization level $\lambda_2$ is not able to certify the safety of the output. This is due to the inherent tradeoff between localization and certification, which we now discuss further.

The optimal $\epsilon$-covers $h(\theta^*)$, $h(\theta^*_{\lambda_1})$, and $h(\theta^*_{\lambda_2})$ are shown in Figure \ref{fig: xiang_net}. The unregularized set $h(\theta^*)$ is massively over-conservative due to the choice $\lambda=0$, which corresponds to pure robustness certification. Indeed, $h(\theta^*)$ is the $\epsilon$-cover from our class of sets that is furthest from the boundary of the safe set, making $\hat{r}(\theta^*)$ the tightest lower bound on $\bar{r}(\epsilon)$. On the other hand, the optimal $\epsilon$-covers using $\lambda=\lambda_1$ and $\lambda=\lambda_2$ are seen to give tighter localizations of the output $Y$. The approximate robustness level using regularization $\lambda_1$ is only slightly lower than the unregularized value, but the regularization $\lambda_2$ is large enough to cause the approximate robustness level $\hat{r}(\theta^*_{\lambda_2})$ to become negative at the expense of localization. This shows how overemphasizing localization may actually harm the certification aspect of robustness assessment, and empirically demonstrates why output set estimation methods may not be adequate for issuing robustness certificates.

\begin{figure}[ht]
	\centering
	\includegraphics[width=0.6\linewidth]{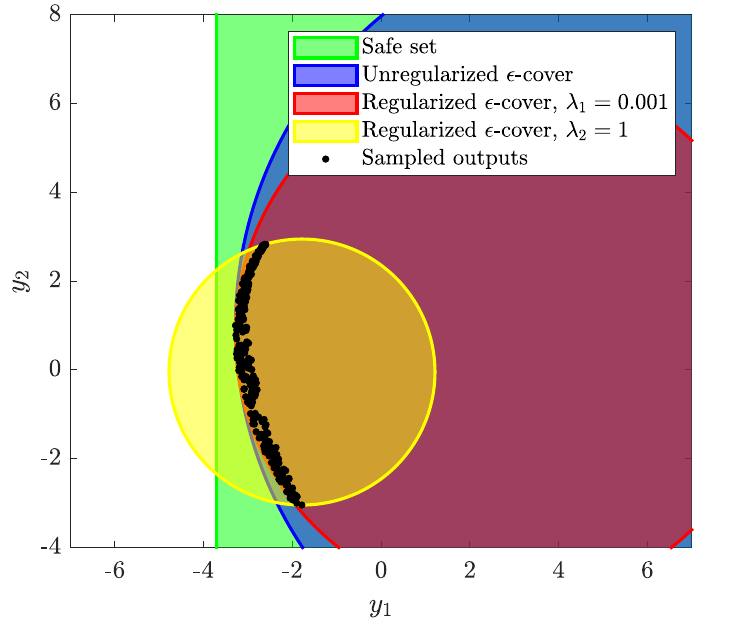}
	\caption{Optimal $\ell_2$-norm ball $\epsilon$-covers for safe set $\mathcal{S}_1$.}
	\label{fig: xiang_net}
\end{figure}

We now repeat the same experiment with a more complicated safe set. In particular, we add an additional constraint to the safe set to match Example 2 given in \cite{xiang2018output}, so that it now takes the form $\mathcal{S}_2 = \{y\in\mathbb{R}^2 : Ay + b \ge 0\}$, where $A = \left[\begin{smallmatrix}
	1 & 0 \\
	-1 & 0
\end{smallmatrix}\right]$ and $b = (3.7, -1.5)$. In this case, we apply our proposed method to each row of the safe set individually. To do so, we set $\epsilon'=\epsilon/2$ and $\delta'=\delta/2$, then define $N' = \lceil \frac{2}{\epsilon'}\left(\log\frac{1}{\delta'}+p\right)\rceil = 609$. For each of the two half-spaces defining the safe set, we solve the scenario problem using $N'$ independent and identically distributed input-output samples, and then we take the intersection of the two resulting $\epsilon'$-covers. Doing so, we obtain an $\epsilon$-cover of the output set with probability at least $1-\delta$. We repeat this process again using regularization levels $\lambda_1=0.001$ and $\lambda_2=1$, and we find that each scenario problem takes approximately $11$ seconds to solve. The resulting covers are shown in Figure \ref{fig: xiang_net_two_halfspaces}.

We find that the approximate robustness levels corresponding to $\lambda=0$ and $\lambda=\lambda_1$ are strictly positive for both half-spaces, certifying that the random output $Y$ is safe with the prescribed probability. However, for $\lambda=\lambda_2$, the optimal $\epsilon$-covers corresponding to both half-spaces are found to intersect the unsafe region of the output space, due to the increased emphasis on localization. Interestingly, the overall localization after intersecting the two $\epsilon'$-covers for $\lambda=\lambda_2$ is in a sense looser than that of the case $\lambda=\lambda_1$, indicating that moderate regularization levels, like $\lambda_1$ in this experiment, may simultaneously perform best for both localization and certification in the case of general polyhedral safe sets defined by more than one half-space. Optimizing $\lambda$ in general poses an interesting problem for future research.

\begin{figure}[ht]
	\centering
	\includegraphics[width=0.6\linewidth]{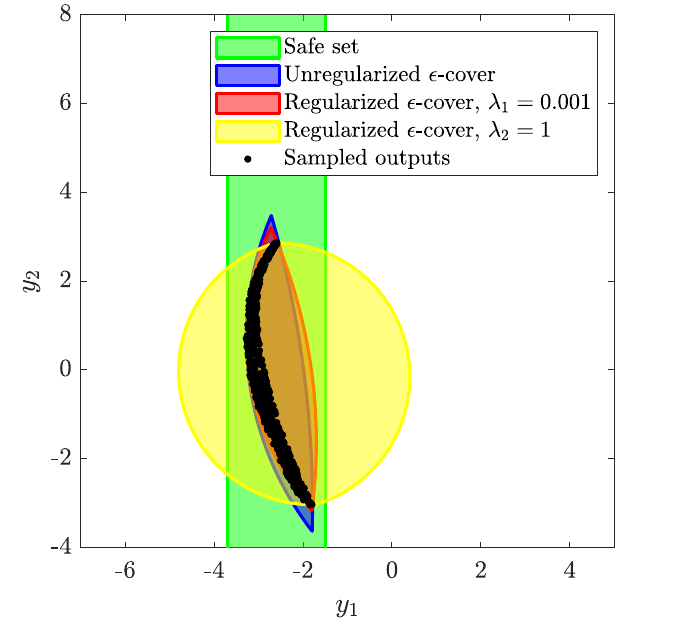}
	\caption{Optimal $\epsilon$-covers amongst intersections of two $\ell_2$-norm ball $\frac{\epsilon}{2}$-covers for safe set $\mathcal{S}_2$.}
	\label{fig: xiang_net_two_halfspaces}
\end{figure}

\section{Comparison to DeepPAC}
\label{sec: comparison_to_deeppac}

In this experiment, we use the half-space special case of our proposed robustness certification method, as presented in Section \ref{sec: special_case-class_of_half-spaces}. Recall that this method uses all optimization efforts to certify the robustness of the network; the outputs are not localized within the output space. For this example, we append a ReLU layer with normal random weights and biases to the neural network presented in Section \ref{sec: illustrative_example}, maintaining $n_y=2$, and we consider the resulting network as a classifier. Ten nominal inputs are chosen randomly at which we will perform robustness certification. The noisy input $X$ is distributed uniformly on the input set $\mathcal{X}=\{x\in\mathbb{R}^2 : \|x-\bar{x}\| \le \epsilon_x\}$, where the radius $\epsilon_x$ is varied from $0.1$ to $1$. We set the probabilistic confidence levels to be $\epsilon=0.1$ and $\delta=10^{-5}$.

For each nominal input and input set radius, we compute a lower bound on the probabilistic robustness level $\bar{r}(\epsilon)$ first using our proposed methodology, and then using the scenario-based approach presented in \cite{li2021probabilistic}, termed DeepPAC.\footnote{Like our robustness certificates, those given by DeepPAC are of the probably approximately correct form (see Remark \ref{rem: pac_learning}), hence the name DeepPAC.} Solving for such a lower bound using DeepPAC first entails solving a scenario linear program for an affine bound on the classifier's margin function, and then requires optimizing this bound over the input set. We remark that DeepPAC requires more samples (and therefore optimization constraints) than our approach, specifically, DeepPAC requires $N\ge \frac{2}{\epsilon}(\log\frac{1}{\delta}+(n_x+1)(n_y-1)+1)$, and therefore we restrict our comparison to DeepPAC to this moderately sized example for computational convenience. See Section \ref{sec: comparison_to_proven} for applications of our approach to large MNIST and CIFAR-10 networks.

After computing the lower bounds on $\bar{r}(\epsilon)$ using both methods, we average the values over the nominal inputs, independently for each input set radius. The results are shown in Figure \ref{fig: deeppac_comparison}. As seen, the lower bounds between the two methods remain close for small input set radii, but our approach offers a tighter lower bound as the radii increase. At input set radius $\epsilon_x=0.4$, our approach is able to issue a high-probability robustness certificate on average, whereas DeepPAC fails. These observations are explained as follows.

DeepPAC works by using samples to learn an affine approximation to the nonlinear margin function over $\mathcal{X}$, so that high-probability bounds on the margin function values at noisy inputs can be made using the learned affine function. Using affine functions to bound the margin function is a technique that naturally applies when considering worst-case or adversarial inputs (e.g., \cite{weng2018towards,zhang2018efficient}). However, as the input set becomes larger, affine approximations are no longer able to accurately capture the nonlinearities of the margin function. Consequently, DeepPAC's high-probability affine bounds on the margin function values become loose, resulting in a looser lower bound on the probabilistic robustness level. Our approach avoids this worst-case analysis technique by directly learning a set in the output space instead of learning a mapping from the input to the output space. This behavior is also seen in the experiment of Section \ref{sec: comparison_to_proven}.

\begin{figure}[ht]
	\centering
	\includegraphics[width=0.6\linewidth]{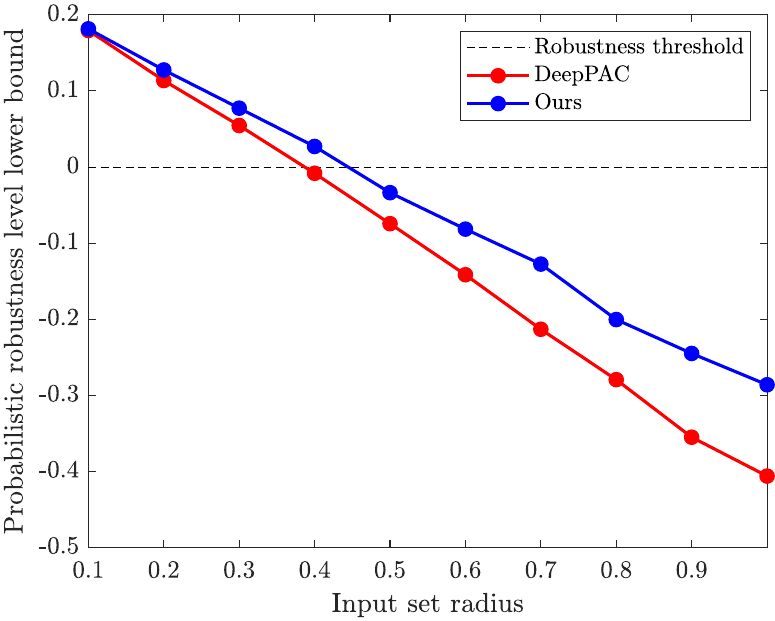}
	\caption{Our lower bound closely matches that of DeepPAC for small input set radii, but becomes noticeably tighter than DeepPAC as the input set becomes larger.}
	\label{fig: deeppac_comparison}
\end{figure}

\section{Comparison to Direct Bayesian Certification}

In this section, we repeat the experiment of Section \ref{sec: comparison_to_proven} using the sample-based certification method given in \cite{zakrzewski2004randomized}. Recall that \cite{zakrzewski2004randomized} imposes a Bayesian framework on the problem by directly assuming that the failure probability follows a uniform prior distribution. In doing so, \cite{zakrzewski2004randomized} is able to certify with probability $1-\delta$ that $\mathbb{P}_X(f(X)\in\mathcal{S}) \ge 1-\epsilon$, so long as the number of samples used is $N \ge \frac{1}{\epsilon}\log\frac{1}{\delta} - 1$ and $f(x_j) \in\mathcal{S}$ for all sampled inputs $x_j$, $j\in\{1,2,\dots,N\}$. In comparing our method to \cite{zakrzewski2004randomized}, two important remarks should be made. First, our method is much more general, as we are able to localize the outputs in arbitrary surrogate output sets, so long as they satisfy the assumptions of Theorem \ref{thm: convex_scenario_optimization}, yielding a convex scenario optimization problem, whereas the method of \cite{zakrzewski2004randomized} is only able to certify whether or not the outputs are contained in the safe half-space. Thus, to compare the methods, we restrict our method to the special case of half-space surrogate output sets from Appendix \ref{sec: special_case-class_of_half-spaces}. Second, for the same amount of samples ($O(\frac{1}{\epsilon}\log\frac{1}{\delta})$), our method provides more information regarding the robustness of the network than \cite{zakrzewski2004randomized} does. In particular, we lower bound the probabilistic robustness level $\bar{r}(\epsilon)$ with a \emph{continuous value}, whereas \cite{zakrzewski2004randomized} is only able to issue a \emph{binary} certificate asserting whether or not outputs are contained in the safe set with high probability. With this in mind, our method is the only one of the two that is able to quantify \emph{how} safe the outputs are by certifying a continuous-valued safety margin.

\begin{table*}[ht]
	\centering
	\caption{Average probabilistic robustness level lower bounds $\hat{r}(\theta^*)$ for MNIST ReLU networks subject to uniform noise over $\ell_\infty$-norm ball. All values are averaged over $10$ nominal inputs with randomly chosen target classes $i$. Also reported are the percentages of inputs that each method is able to certify. Lower bounds giving certified robustness (on average) are bolded, and the average certified adversarial radii computed using \cite{zhang2018efficient} are italicized.}
	\begin{subtable}{0.45\textwidth}
	\centering
	\caption{$2\times[20]$ network.}
	\resizebox{\textwidth}{!}{%
	\begin{tabular}{l r r r r r r}
	\toprule
	\multirow{2}{*}{Radius} & \multicolumn{2}{c}{$\epsilon=0.001$} & \multicolumn{2}{c}{$\epsilon=0.1$} & \multicolumn{2}{c}{$\epsilon=0.25$} \\
	& \cite{zakrzewski2004randomized} & Ours & \cite{zakrzewski2004randomized} & Ours & \cite{zakrzewski2004randomized} & Ours \\
	\midrule %
	\multirow{2}{*}{$0.01$} & $\bf 0.00$ & $\bf 14.11$ & $\bf 0.00$ & $\bf 14.26$ & $\bf 0.00$ & $\bf 14.28$ \\
	& $100$\% & $100$\% & $100$\% & $100$\% & $100$\% & $100$\% \\%
	\multirow{2}{*}{$\mathit{0.027}$} & $\bf 0.00$ & $\bf 13.36$ & $\bf 0.00$ & $\bf 13.77$ & $\bf 0.00$ & $\bf 13.85$ \\
	& $100$\% & $100$\% & $100$\% & $100$\% & $100$\% & $100$\% \\%
	\multirow{2}{*}{$0.05$} & $\bf 0.00$ & $\bf 12.34$ & $\bf 0.00$ & $\bf 13.11$ & $\bf 0.00$ & $\bf 13.26$ \\
	& $100$\% & $100$\% & $100$\% & $100$\% & $100$\% & $100$\% \\%
	\multirow{2}{*}{$0.1$} & $\bf 0.00$ & $\bf 10.25$ & $\bf 0.00$ & $\bf 11.66$ & $\bf 0.00$ & $\bf 12.00$ \\
	& $100$\% & $100$\% & $100$\% & $100$\% & $100$\% & $100$\% \\%
	\multirow{2}{*}{$0.5$} & N/A & $-7.21$ & N/A & $-0.05$ & N/A & $\bf 1.71$ \\
	& $20$\% & $20$\% & $40$\% & $40$\% & $50$\% & $50$\% \\%
	\bottomrule
	\end{tabular}%
	}
	\label{tab: mnist_relu_2x20-zak}
	\end{subtable}
	\hfil
	\begin{subtable}{0.45\textwidth}
	\centering
	\caption{$3\times[20]$ network.}
	\resizebox{\textwidth}{!}{%
	\begin{tabular}{l r r r r r r}
	\toprule
	\multirow{2}{*}{Radius} & \multicolumn{2}{c}{$\epsilon=0.001$} & \multicolumn{2}{c}{$\epsilon=0.1$} & \multicolumn{2}{c}{$\epsilon=0.25$} \\
	& \cite{zakrzewski2004randomized} & Ours & \cite{zakrzewski2004randomized} & Ours & \cite{zakrzewski2004randomized} & Ours \\
	\midrule %
	\multirow{2}{*}{$0.01$} & $\bf 0.00$ & $\bf 17.28$ & $\bf 0.00$ & $\bf 17.45$ & $\bf 0.00$ & $\bf 17.48$ \\
				& $100$\% & $100$\% & $100$\% & $100$\% & $100$\% & $100$\% \\%
	\multirow{2}{*}{$\mathit{0.022}$} & $\bf 0.00$ & $\bf 16.65$ & $\bf 0.00$ & $\bf 17.02$ & $\bf 0.00$ & $\bf 17.10$ \\
	& $100$\% & $100$\% & $100$\% & $100$\% & $100$\% & $100$\% \\%
	\multirow{2}{*}{$0.05$} & $\bf 0.00$ & $\bf 15.19$ & $\bf 0.00$ & $\bf 16.00$ & $\bf 0.00$ & $\bf 16.19$ \\
	& $100$\% & $100$\% & $100$\% & $100$\% & $100$\% & $100$\% \\%
	\multirow{2}{*}{$0.1$} & $\bf 0.00$ & $\bf 12.57$ & $\bf 0.00$ & $\bf 14.19$ & $\bf 0.00$ & $\bf 14.58$ \\
	& $100$\% & $100$\% & $100$\% & $100$\% & $100$\% & $100$\% \\%
	\multirow{2}{*}{$0.5$} & N/A & $-9.36$ & N/A & $-0.55$ & N/A & $\bf 0.36$ \\
	& $0$\% & $0$\% & $50$\% & $50$\% & $62.5$\% & $62.5$\% \\%
	\bottomrule
	\end{tabular}%
	}
	\label{tab: mnist_relu_3x20-zak}
	\end{subtable} \\
	\vspace*{\baselineskip}%
	\begin{subtable}{0.45\textwidth}
	\centering
	\caption{$2\times[1024]$ network.}
	\resizebox{\textwidth}{!}{%
	\begin{tabular}{l r r r r r r}
	\toprule
	\multirow{2}{*}{Radius} & \multicolumn{2}{c}{$\epsilon=0.001$} & \multicolumn{2}{c}{$\epsilon=0.1$} & \multicolumn{2}{c}{$\epsilon=0.25$} \\
	& \cite{zakrzewski2004randomized} & Ours & \cite{zakrzewski2004randomized} & Ours & \cite{zakrzewski2004randomized} & Ours \\
	\midrule %
	\multirow{2}{*}{$0.01$} & $\bf 0.00$ & $\bf 27.73$ & $\bf 0.00$ & $\bf 27.87$ & $\bf 0.00$ & $\bf 27.93$ \\
				& $100$\% & $100$\% & $100$\% & $100$\% & $100$\% & $100$\% \\%
	\multirow{2}{*}{$\mathit{0.032}$} & $\bf 0.00$ & $\bf 26.69$ & $\bf 0.00$ & $\bf 27.13$ & $\bf 0.00$ & $\bf 27.32$ \\
	& $100$\% & $100$\% & $100$\% & $100$\% & $100$\% & $100$\% \\%
	\multirow{2}{*}{$0.05$} & $\bf 0.00$ & $\bf 25.83$ & $\bf 0.00$ & $\bf 26.53$ & $\bf 0.00$ & $\bf 26.82$ \\
	& $100$\% & $100$\% & $100$\% & $100$\% & $100$\% & $100$\% \\%
	\multirow{2}{*}{$0.1$} & $\bf 0.00$ & $\bf 23.46$ & $\bf 0.00$ & $\bf 24.84$ & $\bf 0.00$ & $\bf 25.42$ \\
	& $100$\% & $100$\% & $100$\% & $100$\% & $100$\% & $100$\% \\%
	\multirow{2}{*}{$0.5$} & N/A & $\bf 4.83$ & $\bf 0.00$ & $\bf 11.79$ & $\bf 0.00$ & $\bf 14.45$ \\
	& $80$\% & $80$\% & $100$\% & $100$\% & $100$\% & $100$\% \\%
	\bottomrule
	\end{tabular}%
	}
	\label{tab: mnist_relu_2x1024-zak}
	\end{subtable}
	\hfil
	\begin{subtable}{0.45\textwidth}
	\centering
	\caption{$3\times[1024]$ network.}
	\resizebox{\textwidth}{!}{%
	\begin{tabular}{l r r r r r r}
	\toprule
	\multirow{2}{*}{Radius} & \multicolumn{2}{c}{$\epsilon=0.001$} & \multicolumn{2}{c}{$\epsilon=0.1$} & \multicolumn{2}{c}{$\epsilon=0.25$} \\
	& \cite{zakrzewski2004randomized} & Ours & \cite{zakrzewski2004randomized} & Ours & \cite{zakrzewski2004randomized} & Ours \\
	\midrule %
	\multirow{2}{*}{$0.01$} & $\bf 0.00$ & $\bf 36.86$ & $\bf 0.00$ & $\bf 37.06$ & $\bf 0.00$ & $\bf 37.12$ \\
				& $100$\% & $100$\% & $100$\% & $100$\% & $100$\% & $100$\% \\%
	\multirow{2}{*}{$\mathit{0.024}$} & $\bf 0.00$ & $\bf 35.97$ & $\bf 0.00$ & $\bf 36.44$ & $\bf 0.00$ & $\bf 36.58$ \\
	& $100$\% & $100$\% & $100$\% & $100$\% & $100$\% & $100$\% \\%
	\multirow{2}{*}{$0.05$} & $\bf 0.00$ & $\bf 34.32$ & $\bf 0.00$ & $\bf 35.32$ & $\bf 0.00$ & $\bf 35.59$ \\
	& $100$\% & $100$\% & $100$\% & $100$\% & $100$\% & $100$\% \\%
	\multirow{2}{*}{$0.1$} & $\bf 0.00$ & $\bf 31.10$ & $\bf 0.00$ & $\bf 33.10$ & $\bf 0.00$ & $\bf 33.69$ \\
	& $100$\% & $100$\% & $100$\% & $100$\% & $100$\% & $100$\% \\%
	\multirow{2}{*}{$0.5$} & N/A & $\bf 6.89$ & N/A & $\bf 15.85$ & $\bf 0.00$ & $\bf 18.56$ \\
	& $80$\% & $80$\% & $90$\% & $90$\% & $100$\% & $100$\% \\%
	\bottomrule
	\end{tabular}%
	}
	\label{tab: mnist_relu_3x1024-zak}
	\end{subtable}
	\label{tab: compare_to_zak}
\end{table*}

The results of the experimental comparison are given in Table \ref{tab: compare_to_zak}. It is observed that our method is always able to issue a robustness certificate whenever \cite{zakrzewski2004randomized} does, and furthermore, our lower bounds on the probabilistic robustness level are generally much less conservative than those granted by \cite{zakrzewski2004randomized} (which are binary; either a lower bound of $0$, or failure to issue a certificate altogether). This emphasizes that, for the same number of samples, a special case of our method is strictly more informative as \cite{zakrzewski2004randomized}, and always succeeds in issuing a robustness certificate whenever their method does.

% Acknowledgment

\section*{Acknowledgment}

This work was supported by grants from AFOSR, ONR, and NSF. The authors would like to thank Lily Weng for their insightful discussions on the implementation of PROVEN.

% Bibliography.
\bibliographystyle{IEEEtran}
\bibliography{references}

% Biographies.
\begin{comment}
\begin{IEEEbiography}[{\includegraphics[width=1in,height=1.25in,clip,keepaspectratio]{example_headshot.png}}]{Brendon G.\ Anderson} is\dots
\end{IEEEbiography}

\begin{IEEEbiography}[{\includegraphics[width=1in,height=1.25in,clip,keepaspectratio]{example_headshot.png}}]{Somayeh Sojoudi} is \dots
\end{IEEEbiography}
\end{comment}

\end{document}